\documentclass[final,12pt]{colt2021}

\usepackage[utf8]{inputenc} 
\usepackage[T1]{fontenc}    
\usepackage{hyperref}       
\usepackage{url}            
\usepackage{booktabs}       
\usepackage{amsfonts}       
\usepackage{nicefrac}       
\usepackage{microtype}      
\usepackage{caption}
\usepackage{comment}
 
\usepackage[toc,page,header]{appendix}
\usepackage{minitoc}
\renewcommand \thepart{}
\renewcommand \partname{}

\usepackage{makecell}
 
\usepackage{amsmath,dsfont,empheq}
\usepackage{algorithm}
\usepackage{algpseudocode}
\usepackage{syntonly,bm,wrapfig}
\usepackage{xcolor,cases,balance}
\input{mysymbol.sty}

\DeclareMathOperator*{\argmin}{arg\,min}

\newcommand{\EE}{\mathbb{E}}
\DeclarePairedDelimiter{\dotp}{\langle}{\rangle}

\DeclareMathOperator{\Var}{Var}

\def \bbeta {{\boldsymbol \beta}}

\newcommand\numberthis{\addtocounter{equation}{1}\tag{\theequation}}

\newtheorem{lemma}{Lemma}
\newtheorem{proposition}{Proposition}

\newtheorem{innercustomthm}{Theorem}
\newenvironment{customthm}[1]
  {\renewcommand\theinnercustomthm{#1}\innercustomthm}
  {\endinnercustomthm}
\newcommand{\E}{\mathbb{E}}

\makeatletter
\def\thanks#1{\protected@xdef\@thanks{\@thanks
        \protect\footnotetext{#1}}}
\makeatother

\algrenewcommand{\algorithmiccomment}[1]{\hskip0em$\triangleright$ #1}
\newcommand{\FullTitle}{Tighter Analysis of Alternating Stochastic Gradient Method for Stochastic Nested Problems}

\title[\FullTitle]{\FullTitle}

\usepackage{times}

\coltauthor{%
\Name{Tianyi Chen} \Email{chent18@rpi.edu}\\
  \addr Rensselaer Polytechnic Institute\\
  \Name{Yuejiao Sun} \Email{sunyj@math.ucla.edu}\\
  \addr   University of California, Los Angeles\\
  \Name{Wotao Yin} \Email{	wotaoyin@math.ucla.edu}\\
  \addr   University of California, Los Angeles
 \thanks{
 The work of T. Chen was partially supported by National
Science Foundation under the project NSF 2047177 and the RPI-IBM Artificial Intelligence Research Collaboration (AIRC). The work of Y. Sun was partially supported by ONR Grant N000141712162 and AFOSR MURI FA9550-18-1-0502.}
	}

\begin{document}

\maketitle

\begin{abstract}
Stochastic nested optimization, including stochastic compositional, min-max and bilevel optimization, is gaining popularity in many machine learning applications. 
While the three problems share the nested structure, existing works often treat them separately, and thus develop problem-specific algorithms and their analyses. 
Among various exciting developments, simple SGD-type updates (potentially on multiple variables) are still prevalent in solving this class of nested problems, but they are believed to have slower convergence rate compared to that of the non-nested problems. 
This paper unifies several SGD-type updates for stochastic nested problems into a single SGD approach that we term ALternating Stochastic gradient dEscenT (ALSET) method. By leveraging the \emph{hidden smoothness} of the problem, this paper presents a tighter analysis of ALSET for stochastic nested problems. 
Under the new analysis, to achieve an $\epsilon$-stationary point of the nested problem, it requires ${\cal O}(\epsilon^{-2})$ samples. 
Under certain regularity conditions, applying our results to stochastic compositional, min-max and reinforcement learning problems either improves or matches the best-known sample complexity in the respective cases. 
Our results explain \emph{why simple SGD-type algorithms in stochastic nested problems all work very well in practice without the need for further modifications.}

\end{abstract}

\section{Introduction}
Stochastic gradient descent (SGD) methods \cite{robbins1951} are prevalent in solving large-scale machine learning problems. 
Oftentimes, SGD is being applied to solve stochastic problems with a relatively simple structure. Specifically, applying SGD to minimize the function $\mathbb{E}_{\xi}\left[f(x;\xi)\right]$ over the variable $x\in \mathbb{R}^d$, we have the iterative update $x^{k+1}=x^k - \alpha \nabla f(x^k;\xi^k)$, where $\alpha>0$ is the stepsize and $\nabla f(x^k;\xi^k)$ is the stochastic gradient at the iterate $x^k$ and the sample $\xi^k$. However, many problems in machine learning today, such as meta learning, deep learning, hyper-parameter optimization, and reinforcement learning, go beyond the above simple minimization structure (termed the non-nested problem thereafter). For example, the objective function may be the compositions of multiple functions, where each composition may introduce an additional expectation \cite{finn2017icml}; and, the objective function may depend on the solution of another optimization problem \cite{franceschi2018bilevel}. In these problems, how to apply SGD and what is the efficiency of running SGD is not fully-understood. 

To answer these questions, in this paper, we consider the following form of \emph{stochastic nested optimization problems}, which is a generalization of the non-nested problems, given by 
\begin{subequations}\label{opt0}
	\begin{align}
	&\min_{x\in \mathbb{R}^d}~~~F(x):=\mathbb{E}_{\xi}\left[f\left(x, y^*(x);\xi\right)\right]~~~~~~~~~~~~~~~~~{\rm\sf (upper)} \\
	&~{\rm s. t.}~~~~~y^*(x)= \argmin_{y\in \mathbb{R}^{d'}}~\mathbb{E}_{\phi}[g(x, y;\phi)]~~~~~~~~~~~~{\rm\sf (lower)} \label{opt0-low}
\end{align} 
\end{subequations}
where $f$ and $g$ are differentiable functions; and, $\xi$ and $\phi$ are random variables. 
In the optimization literature \cite{colson2007overview,kunapuli2008,dempe2020bilevel}, the problem \eqref{opt0} is referred to as the stochastic \emph{bilevel} problem, where the upper-level optimization problem depends on the solution of the lower-level optimization over $y\in  \mathbb{R}^{d'}$, denoted as $y^*(x)$, which depends on the value of upper-level variable $x\in  \mathbb{R}^d$. 

The stochastic \emph{bilevel} nested problem \eqref{opt0} encompasses two popular formulations with the nested structure: stochastic min-max problems and stochastic compositional problems. 
Therefore, results on the general nested problem \eqref{opt0} will also imply the results in the special cases.  
For example, if the lower-level objective $g$ is the negative of the upper-level objective $f$, i.e., $g(x,y;\phi):=-f(x,y;\xi)$, the stochastic bilevel problem \eqref{opt0}  reduces to the stochastic min-max problem
\begin{align}\label{opt0-1}
\text{If $g(x,y;\phi):=-f(x,y;\xi)$}~~~~\Rightarrow ~~~~ \min_{x\in\mathbb{R}^d}~F(x):=\max_{y\in\mathbb{R}^{d'}}~\mathbb{E}_{\xi}\left[f(x,y;\xi)\right]. 
\end{align}
Motivated by applications in zero-sum games, adversarial learning and training GANs, significant efforts have been made for solving the stochastic min-max problem; see e.g., \cite{daskalakis2018limit,gidel2018variational,rafique2018non,thekumparampil2019efficient,mokhtari2020unified}. 

For example, if the upper-level objective $f$ is only a function of $y$, i.e., $f(x,y;\xi):=f(y;\xi)$, and the lower-level objective $g$ is a quadratic function of $y$, i.e., $g(x,y;\phi):=\|y-h(x;\phi)\|^2$ with a smooth function $h$ of $x$, then the variable $y^*(x)$ admits a closed-form solution, and thus the stochastic bilevel problem \eqref{opt0} reduces to the stochastic compositional problem \cite{wang2017mp,dai2017learning,ghadimi2020jopt}
\begin{align}\label{opt0-2}
\text{If $g(x,y;\phi):=\|y-h(x;\phi)\|^2$}~~~~\Rightarrow ~~~~	\min_{x\in \mathbb{R}^d}~F(x):=\mathbb{E}_{\xi}\left[f\big(\mathbb{E}_{\phi}[h(x;\phi)];\xi\big)\right]. 
\end{align} 
Stochastic compositional problems in the form of \eqref{opt0-2} have been studied in the applications in model-agnostic meta learning and policy evaluation in reinforcement learning; see e.g., \cite{finn2017icml,ji2020feb}.    

To solve the nested problem \eqref{opt0} by SGD, one natural solution is to apply \emph{alternating SGD updates} on $x$ and $y$ based on their stochastic gradients
\begin{equation}\label{eq.sgd2}
	 y^{k+1} =y^k -\beta_k h_g^k ~~~~~~{\rm and}~~~~~~x^{k+1}=x^k - \alpha_k h_f^k
\end{equation}
where $h_g^k$ is the unbiased stochastic gradient of $\mathbb{E}_{\phi}[g(x^k, y^k;\phi)]$ and $h_f^k$ is the (possibly biased) stochastic gradient of $F(x^k)$; and, $\beta_k$ and $\alpha_k$ are the stepsizes. 
A key challenge of running \eqref{eq.sgd2} for the nested problem is that (stochastic) gradient of the upper-level variable $x$ is prohibitively expensive to compute. As we will show later, computing an unbiased stochastic gradient of $F(x)$ requires solving the lower-level problem exactly to obtain $y^*(x)$. 

To obtain an accurate stochastic gradient $h_f^k$, there are roughly three ways in the literature. 
One way is to run SGD updates on $y^k$ multiple times before updating $x^k$, which yields a double-loop algorithm. To guarantee convergence, it typically requires either the increasing number of lower-level $y$-update or the increasing number of batch size to estimate $h_g^k$; see e.g., \cite{ghadimi2018bilevel,ji2020provably}. 
The second way is to update $y^k$ in a timescale faster than that of $x^k$ so that $x^k$ is relatively static with respect to $y^k$; i.e., $\lim_{k\rightarrow \infty}\alpha_k/\beta_k=0$; see e.g., \cite{hong2020ac}. 
The third way is to modify the direction $h_g^k$ of $y^k$ by incorporating additional correction term, which adds extra computation burden; see e.g., \cite{chen2021single}.
At a high level, these modifications either deviate from the originally light-weight implementation of SGD or sacrifice the sample complexity of SGD. 

To this end, the \textbf{main goal} of this paper is to study the efficiency of running the vanilla alternating SGD \eqref{eq.sgd2} for the nested problem \eqref{opt0} and its implications on the special problem classes \eqref{opt0-1}-\eqref{opt0-2}.



\subsection{Main results}
This paper aims to analyze a unifying algorithm for the stochastic bilevel problems that runs SGD on each variable in an alternating fashion, and provide sample complexity that matches the complexity of SGD for single-level stochastic problems. Our results explain why SGD-type algorithms in stochastic bilevel, min-max, and compositional problems all work very well in practice without the need for modifications such as correction, increasing batch size and two-timescale stepsizes. 

In the context of existing methods, our contributions can be summarized as follows. 

\begin{enumerate}
	\item [C1)] We connect three different classes of stochastic nested optimization problems (namely, stochastic compositional, min-max, and bilevel optimization), and unify three popular SGD-type updates for the respective problems into a single SGD-type approach that we term ALternating Stochastic gradient dEscenT (ALSET) method. 
		
	\item [C2)] Under the same assumptions made in most of the previous work, we discover that the solution of the lower-level problem is smooth -- a property that is overlooked by the previous analyses. By leveraging the \emph{hidden smoothness}, we present a tighter analysis of ALSET for the stochastic bilevel problems. Under the new analysis, to achieve an $\epsilon$-stationary point of the nested problem, ALSET requires ${\cal O}(\epsilon^{-2})$ samples in total, rather than the ${\cal O}(\epsilon^{-2.5})$ sample complexity in the existing literature. 
	
     \item [C3)] We further customize the analysis to the two special cases -- the compositional and min-max problems, and establish the improved sample complexity relative to that in the literature. We apply our a new analysis to the celebrated actor-critic method for reinforcement learning problems. Under some regularity conditions, our new analysis implies that to achieve an $\epsilon$-stationary point, the single-loop actor-critic method requires ${\cal O}(\epsilon^{-2})$ samples with i.i.d. sampling, which improves the ${\cal O}(\epsilon^{-2.5})$ sample complexity in the existing literature.  
 
 \end{enumerate}
 
\subsection{Other related works}
To put our work in context, we review prior art that we group in the following three categories.

\noindent\textbf{Stochastic bilevel optimization.} 
The study of bilevel optimization can be traced back to 1950s \citep{stackelberg}. 
Many recent efforts have been made to solve the bilevel problems.
One successful approach is to reformulate the bilevel problem as a single-level problem by replacing the lower-level problem by its optimality conditions~\citep{colson2007overview,kunapuli2008}. Recently, gradient-based methods for bilevel optimization have gained popularity, where the idea is to iteratively approximate the (stochastic) gradient of the upper-level problem either in forward or backward manner \citep{sabach2017jopt,franceschi2018bilevel,shaban2019truncated,grazzi2020iteration}. 
Recent work has also studied the case where the lower-level problem does not have a unique solution \citep{liu2020icml}.

The non-asymptotic analysis of bilevel optimization algorithms has been recently studied in some \emph{pioneering} works, e.g., \citep{ghadimi2018bilevel,hong2020ac,ji2020provably}, just to name a few. In both \citep{ghadimi2018bilevel,ji2020provably}, bilevel stochastic optimization algorithms have been developed that run in a double-loop manner.
To achieve an $\epsilon$-stationary point, they only need the sample complexity ${\cal O}(\epsilon^{-2})$ that is comparable to the complexity of SGD for the single-level case. 
Recently, a single-loop two-timescale stochastic approximation algorithm has been developed in \citep{hong2020ac} for the bilevel problem \eqref{opt0}. Due to the nature of two-timescale update, it incurs the sub-optimal sample complexity ${\cal O}(\epsilon^{-2.5})$. 
A single-loop single-timescale stochastic bilevel optimization method has been recently developed in \cite{chen2021single}. While the method can achieve the sample complexity ${\cal O}(\epsilon^{-2})$, the resultant update on $y$ needs extra matrix projection, which can be costly. Very recently, the momentum-based acceleration has been incorporated into the $x$- and $y$-updates in \cite{khanduri2021momentum,guo2021randomized} and in \cite{yang2021provably} during our preparation of the online version, where the new algorithms enjoy an improved sample complexity ${\cal O}(\epsilon^{-3/2})$. However, these results cannot imply the ${\cal O}(\epsilon^{-2})$ sample complexity of the alternating SGD update \eqref{eq.sgd2}, and are orthogonal to our results. A comparison of our results with prior work can be found in Table \ref{table:comp1}.

 \begin{table*}[t]
 \small
	\begin{center}
		\begin{tabular}{c||c|c | c|c| c|c }
			\hline
			\hline
 	& {ALSET}  & {BSA} &  {TTSA}  & {stocBiO} & {STABLE} &\!\!\!{SUSTAIN/RSVRB} \!\!\! \\ \hline
			\textbf{batch size} & ${\cal O}(1)$   & ${\cal O}(1)$   & ${\cal O}(1)$  & ${\cal O}(\epsilon^{-1})$   & ${\cal O}(1)$& ${\cal O}(1)$   \\ \hline
		 \textbf{$y$-update} & SGD  &\!\! ${\cal O}(\epsilon^{-1})$ SGD steps \!\! & SGD   & SGD  & correction & momentum    \\ \hline
		 \makecell{\makecell{\textbf{samples in $\xi$}\\ \textbf{samples in $\phi$}}}   & \makecell{ ${\cal O}(\epsilon^{-2})$ \\ ${\cal O}(\epsilon^{-2})$ } & \makecell{ ${\cal O}(\epsilon^{-2})$  \\ ${\cal O}(\epsilon^{-3})$ }  & \makecell{ ${\cal O}(\epsilon^{-2.5})$  \\ ${\cal O}(\epsilon^{-2.5})$ }& \makecell{ ${\cal O}(\epsilon^{-2})$  \\ ${\cal O}(\epsilon^{-2})$ } & \makecell{ ${\cal O}(\epsilon^{-2})$  \\ ${\cal O}(\epsilon^{-2})$ }& \makecell{ ${\cal O}(\epsilon^{-3/2})$   \\ ${\cal O}(\epsilon^{-3/2})$ } \\ \hline
		 	\hline
		\end{tabular}
	\end{center}
 	\vspace{-0.2cm}
   \caption{Sample complexity of stochastic bilevel algorithms (BSA in \citep{ghadimi2018bilevel}, TTSA in \citep{hong2020ac}, stocBiO in \citep{ji2020provably}, STABLE in \cite{chen2021single}, SUSTAIN in \cite{khanduri2021momentum}, RSVRB in \cite{guo2021randomized}) to achieve an $\epsilon$-stationary point of $F(x)$.}			\label{table:comp1}
   \vspace{-0.3cm}
\end{table*}

\begin{figure*}[b]
\vspace{-0.2cm}
 \small
\begin{minipage}[c]{0.495\textwidth}
		\begin{tabular}{c||c|c | c   }
			\hline
			\hline
 	& {ALSET}  & {	SGDA} &  {SMD}      \\ \hline
			\textbf{batch size} & ${\cal O}(1)$   & ${\cal O}(\epsilon^{-1})$  & /   \\ \hline
		 \textbf{$y$-update} & SGD  & SGD& subproblem         \\ \hline
		\textbf{samples }  &  ${\cal O}(\kappa^3\epsilon^{-2})$  & ${\cal O}(\kappa^3\epsilon^{-2})$  &  ${\cal O}(\kappa^3\epsilon^{-2})$\\ \hline
		 	\hline
		\end{tabular}
\captionof{table}{Sample complexity of stochastic min-max algorithms (BSA in \citep{ghadimi2018bilevel}, GDA in \citep{lin2020gradient}, SMD in \citep{rafique2018non}) to achieve an $\epsilon$-stationary point of $F(x)$.}			\label{table:comp2}
\end{minipage} 
\hspace{0.5cm}
\begin{minipage}[c]{0.495\textwidth}
		\begin{tabular}{c||c|c | c   }
			\hline
			\hline
 	& {ALSET}  & {SCGD} &  {NASA}      \\ \hline
			\textbf{batch size} & ${\cal O}(1)$   & ${\cal O}(1)$  & ${\cal O}(1)$     \\ \hline
		 \textbf{$y$-update} & SGD  & SGD& correction         \\ \hline
 \textbf{samples }  &  ${\cal O}(\epsilon^{-2})$  & ${\cal O}(\epsilon^{-4})$  & ${\cal O}(\epsilon^{-2})$ \\ \hline
		 	\hline
		\end{tabular}
\captionof{table}{Sample complexity of stochastic compositional algorithms (SCGD in \citep{wang2017mp}, NASA in \citep{ghadimi2020jopt}) to achieve an $\epsilon$-stationary point of $F(x)$.}			\label{table:comp3}
\end{minipage} 	
\vspace{-0.2cm}
\end{figure*}

\noindent\textbf{Stochastic min-max optimization.}
In the context of min-max problems, the alternating version of the stochastic gradient descent ascent (GDA) method can be viewed as the alternating SGD updates \eqref{eq.sgd2} for the special nested problem \eqref{opt0-1}. 
To mitigate the cycling behavior of GDA for convex-concave min-max problems, several variants have been developed by incorporating the idea of optimism; see e.g., \cite{daskalakis2018limit,gidel2018variational,mokhtari2020unified,yoon2021accelerated}. 
The analysis of stochastic GDA in the nonconvex-strongly concave setting is closely related to this paper; e.g., 
\cite{rafique2018non,thekumparampil2019efficient,nouiehed2019solving,lin2020gradient}. 
Specifically, for stochastic GDA (SGDA), the ${\cal O}(\epsilon^{-2})$ sample complexity has been established in \cite{lin2020gradient} under an increasing batch size ${\cal O}(\epsilon^{-1})$. 
As highlighted in \cite{lin2020gradient}, how to achieve the ${\cal O}(\epsilon^{-2})$ sample complexity under an ${\cal O}(1)$ constant batch size remains open. The reduction of our results to the min-max setting will provide an answer to this open question. In the same setting, accelerated GDA algorithms have been developed in \cite{luo2020stochastic,yan2020optimal,tran2020hybrid}.    
Going beyond the one-side concave settings, algorithms and their convergence analysis have been studied for nonconvex-nonconcave min-max problems with certain benign structure; see e.g., \cite{gidel2018variational,liu2019towards,yang2020global,diakonikolas2021efficient}.
A comparison of our results with prior work can be found in Table \ref{table:comp2}.

\noindent\textbf{Stochastic compositional optimization.}
Stochastic compositional gradient algorithms developed in \citep{wang2017mp,wang2017jmlr} can be viewed as the alternating SGD updates \eqref{eq.sgd2} for the special compositional problem \eqref{opt0-2}. However, to ensure convergence, the algorithms \citep{wang2017mp,wang2017jmlr} use two sequences of variables being updated in two different time scales, and thus 
the complexity of \citep{wang2017mp} and \citep{wang2017jmlr} is worse than ${\cal O}(\epsilon^{-2})$ of SGD for the non-compositional case.  
While most of existing algorithms rely on either two-timescale updates, the single-timescale  single-loop approaches have been recently developed in \citep{ghadimi2020jopt,chen2020scsc,ruszczynski2020stochastic}, which achieve the sample complexity ${\cal O}(\epsilon^{-2})$, same as SGD for the non-nested problems. However, the algorithms proposed therein are not the vanilla alternating SGD update in the sense of \eqref{eq.sgd2}. Other related compositional algorithms also include \citep{lian2017aistats,zhang2019nips,tran2020stochastic}.
A comparison can be found in Table \ref{table:comp3}.

\textbf{Organization.}
The basic background of bilevel optimization is reviewed, and the
tighter analysis of the unifying ALSET method is presented
in Section \ref{sec.scg}. 
The reduction of the main results to the special stochastic nested problems are provided in Section \ref{sec.special}, and its applications to the actor-critic method is discussed in Section 4, followed by the conclusions in Section 5.

\section{Improved Analysis of Alternating Stochastic Gradient Method}
\label{sec.scg}
 
In this section, we will first provide background of bilevel problems and then introduce the general alternating stochastic  gradient descent (ALSET) method for stochastic nested problems.  

\vspace{-0.1cm}
\subsection{Preliminaries}
We use $\|\cdot \|$ to denote the $\ell_2$ norm for vectors and Frobenius norm for matrices.
For convenience, we define the deterministic functions as $g(x, y):=\mathbb{E}_{\phi}[g(x, y;\phi)]$ and $f(x, y):=\mathbb{E}_{\xi}[f(x, y;\xi)]$.

We also define $\nabla_{yy}^2 g\big(x, y\big)$ as the Hessian matrix of $g$ with respect to $y$ and define $  \nabla_{xy}^2g\big(x, y\big)$ as
\begin{align*}
  \nabla_{xy}^2g\big(x, y\big) :=    \begin{bmatrix}
    \frac{\partial^2}{\partial x_1\partial y_1 }g\big(x, y\big) & \cdots & \frac{\partial^2}{\partial x_1\partial y_{d'}}g\big(x, y\big)\\
    & \cdots & \\
    \frac{\partial^2}{\partial x_d\partial y_1 }g\big(x, y\big) & \cdots & \frac{\partial^2}{ \partial x_d\partial y_{d'}}g\big(x, y\big)
  \end{bmatrix}.
\end{align*}


We make the following assumptions that are common in bilevel optimization  \citep{ghadimi2018bilevel,ji2020provably,hong2020ac,guo2021randomized}. 

\vspace{0.1cm}
\begin{assumption}[Lipschitz continuity]
Assume that $f, \nabla f, \nabla g, \nabla^2g$ are respectively $\ell_{f,0}$, $\ell_{f,1}, \ell_{g,1}, \ell_{g,2}$-Lipschitz continuous; that is, for $z_1:=[x_1;y_1]$, $z_2:=[x_2;y_2]$, we have $\|f(x_1,y_1)-f(x_2,y_2)\|\leq \ell_{f,0}\|z_1-z_2\|, \|\nabla f(x_1,y_1)-\nabla f(x_2,y_2)\|\leq \ell_{f,1}\|z_1-z_2\|, \|\nabla g(x_1,y_1)-\nabla g(x_2,y_2)\|\leq \ell_{g,1}\|z_1-z_2\|, \|\nabla^2 g(x_1,y_1)-\nabla^2 g(x_2,y_2)\|\leq \ell_{g,2}\|z_1-z_2\|$.
\end{assumption}

\begin{assumption}[Strong convexity of $g$ in $y$]
For any fixed $x$, $g(x,y)$ is $\mu_g$-strongly convex in $y$.
\end{assumption}

Assumptions 1 and 2 together ensure that the first- and second-order derivations of $f(x,y), g(x,y)$ as well as the solution mapping $y^*(x)$ are well-behaved. Define the condition number $\kappa:={\ell_{g,1}}/{\mu_g}$.

\begin{assumption}[Stochastic derivatives]
The stochastic derivatives $\nabla f(x,y;\xi)$, $\nabla g(x,y;\phi)$, $\nabla^2g(x, y, \phi)$ are unbiased estimators of $\nabla f(x,y)$, $\nabla g(x,y)$, $\nabla^2g(x, y)$, respectively; and their variances are bounded by $\sigma_f^2, \sigma_{g,1}^2$, $\sigma_{g,2}^2$, respectively.
\end{assumption}

Assumptions 2 and 3 together imply that the second moments are bounded by 
\begin{subequations}
\begin{align}\label{eq.assp_grad_norm}
&\EE_{\xi}[\|\nabla f(x,y;\xi)\|^2]\leq \ell_{f,0}^2 + \sigma_f^2:=C_{f}^2\\
&\EE_{\phi}[\|\nabla^2 g(x,y;\phi)\|^2]\leq \ell_{g,1}^2 + \sigma_{g,2}^2:=C_{g}^2.
\end{align}
\end{subequations}

Assumption 3 is the counterpart of the unbiasedness and bounded variance assumption in the single-level stochastic optimization. In addition, the bounded moments in Assumption 3 ensure the Lipschitz continuity of the upper-level gradient $\nabla F(x)$.



We first highlight the inherent challenge of directly applying the alternating SGD method to the bilevel problem \eqref{opt0}. 
To illustrate this point, we derive the gradient of the upper-level function $F(x)$ in the next proposition; see the proof in the supplementary document. 
\begin{proposition}\label{prop1}
Under Assumptions 1--3, we have the gradients
	\begin{align}\label{grad-deter-2}
 \nabla F(x)=\nabla_xf(x, y^*(x))-\nabla_{xy}^2g(x, y^*(x))\!\left[\nabla_{yy}^2 g(x, y^*(x))\right]^{-1}\nabla_y f(x, y^*(x)).
\end{align}
Furthermore, $\nabla F(x)$ and $y^*(x)$ are Lipschitz continuous with constants $L_F, L_y$, respectively. 
\end{proposition}
Notice that obtaining an unbiased stochastic estimate of $\nabla F(x)$ and applying SGD on $x$ face two main difficulties: i) the gradient $\nabla F(x)$ at $x$ depends on the minimizer of the lower-level problem $y^*(x)$; ii) even if $y^*(x)$ is known, it is hard to apply the stochastic approximation to obtain an unbiased estimate of $\nabla F(x)$ since $\nabla F(x)$ is nonlinear in $\nabla_{yy}^2 g(x, y^*(x))$. 

Similar to some existing stochastic bilevel algorithms \cite{ghadimi2018bilevel,hong2020ac,ji2020provably}, we evaluate $\nabla F(x)$ on a certain vector $y$ in place of $y^*(x)$.
Replacing the $y^*(x)$ in definition \eqref{grad-deter-2} by $y$, we define
\begin{equation}\label{grad-deter-3}
 \overline{\nabla}_xf\big(x, y\big):=\nabla_xf\big(x, y\big) -\nabla_{xy}^2g\big(x, y\big)\left[\nabla_{yy}^2 g\big(x, y\big)\right]^{-1}\nabla_y f\big(x, y\big). 
\end{equation}

And to reduce the bias in \eqref{grad-deter-3}, we estimate $\left[\nabla_{yy}^2 g\big(x, y\big)\right]^{-1}$ via
\begin{equation}\label{eq.inverse-estimator}
\left[\nabla_{yy}^2 g\big(x, y\big)\right]^{-1}\approx \Big[\frac{N}{\ell_{g,1}}\prod\limits_{n=1}^{N'}\Big(I-\frac{1}{\ell_{g,1}}\nabla_{yy}^2g(x,y;\phi_{(n)})\Big)\Big]
\end{equation}
where $N'$ is drawn from $\{1, 2, \ldots, N\}$ uniformly at random and $\{\phi^{(1)}, \ldots,\phi^{(N')}\}$ are i.i.d. samples.
It has been shown in \cite{ghadimi2018bilevel} that using \eqref{eq.inverse-estimator}, the estimation bias of $\left[\nabla_{yy}^2 g\big(x, y\big)\right]^{-1}$ exponentially decreases with the number of samples $N$.

\begin{wrapfigure}{R}{0.6\textwidth}
\vspace{-0.8cm}
 \begin{minipage}{0.6\textwidth}
\begin{algorithm}[H]
\caption{ALSET for stochastic bilevel problems}\label{alg: ALSET}
\begin{algorithmic}[1]
    \State{\textbf{initialize:} $x^0, y^0$, stepsizes $\{\alpha_k, \beta_k\}$.} 
        \For{$k=0,1,\ldots, K-1$}
            \For{$t=0,1,\ldots,T-1$}
                \State{update $y^{k,t+1}=y^{k,t}-\beta_kh_g^{k,t}$}~~~\Comment{set $y^{k,0}=y^k$}
            \EndFor
            \State{update $x^{k+1}=x^k-\alpha h_f^k$}     ~~~~~~~~~       \Comment{set $y^{k+1}=y^{k,T}$}
        \EndFor
\end{algorithmic}
\end{algorithm}
 \end{minipage}
 \vspace{-0.4cm}
\end{wrapfigure}

\subsection{Main results: Tighter analysis of ALSET}
In this subsection, we first describe the general ALSET algorithm for the stochastic bilevel problem, and then present its new convergence result.

This algorithm is very simple to implement. 
At each iteration $k$, ALSET alternates between the stochastic gradient update on $y^k$ and that on $x^k$. 
Although it is possible that $T=1$, for generality, we run $T$ steps of SGD on the lower-level variable $y^k$ before updating upper-level variable $x^k$. With $\alpha_k$ and $\beta_k$ denoting the stepsizes of $x^k$ and $y^k$ that decrease at the same rate as SGD, the ALSET update is  
\begin{subequations}\label{eq.ALSET}
\begin{align}
\!\!\!\!    y^{k,t+1}&=y^{k,t} \!-\beta_k h_g^{k,t},\,t=0,\ldots, T~~~~{\rm with}~~~y^{k,0}:=y^k;~ y^{k+1}:=y^{k,T} \label{eq.ALSET3}\\
\!\!\!\!    x^{k+1}&=x^k \!-\! \alpha_k h_f^k\label{eq.ALSET2}
\end{align}
\end{subequations}
where the update direction of $y$ is the stochastic gradient $h_g^{k,t} := \nabla_yg(x^k,y^{k,t};\phi^{k,t})$; and, with the Hessian inverse estimator \eqref{eq.inverse-estimator}, the update direction of $x$ is the slightly biased gradient
\begin{align}\label{eq.biased-estimator}
h_f^k:=&\nabla_x f(x^k,y^{k+1};\xi^k)\nonumber\\
&-\nabla_{xy}^2g(x^k,y;\phi_{(0)}^k)\Big[\frac{N}{\ell_{g,1}}\prod\limits_{n=1}^{N'}\Big(I-\frac{1}{\ell_{g,1}}\nabla_{yy}^2g(x^k,y^{k+1};\phi_{(n)}^k)\Big)\Big]\nabla_yf(x^k,y^{k+1};\xi^k).
\end{align}

The alternating update \eqref{eq.ALSET} serves as a template of running SGD on stochastic nested problems. As we will show in the subsequent sections, we can generate stochastic algorithms for min-max,  compositional, and even reinforcement learning problems following \eqref{eq.ALSET} as a template, but they differ in the particular forms of the stochastic gradients $h_g^k, h_f^k$ for the specific upper- and lower-level objective functions.  
See a summary of ALSET for the bilevel problem in Algorithm \ref{alg: ALSET}.

\textbf{Comparison between ALSET with existing works.}
Readers who are familiar with recent developments on stochastic optimization for bilevel problems may readily recognize the similarities between the general ALSET update \eqref{alg: ALSET} that we will analyze and the SGD-based updates in BSA  \citep{ghadimi2018bilevel}, TTSA \citep{hong2020ac} and stocBiO \citep{ji2020provably}. 
However, the update \eqref{alg: ALSET} is different from BSA in that the number of $y$-update, denoted as $T$, is a constant in \eqref{alg: ALSET} that does not grow with the accuracy $\epsilon^{-1}$; the  update \eqref{alg: ALSET} is different from stocBiO in that the stochastic gradient $h_g^{k,t}$ used in the $y$-update \eqref{eq.ALSET3} is obtained by a fixed batch size that does not depend on the accuracy $\epsilon^{-1}$; and, the update \eqref{alg: ALSET} is different from TTSA in that the stepsizes $\alpha_k$ and $\beta_k$ in \eqref{eq.ALSET} decrease at the same timescale.


We next present the convergence result of ALSET. 
\begin{customthm}{1}[Nonconvex]\label{theorem1}
Under Assumptions 1--3, define the constants as 
\begin{equation}
 \bar\alpha_1=\frac{1}{2L_F+4L_fL_y+\frac{L_fL_{yx}}{L_y\eta}}, ~~~\bar\alpha_2=\frac{16T\mu_g\ell_{g,1}}{(\mu_g+\ell_{g,1})^2(8L_fL_y + \eta L_{yx}\tilde{C}_f^2\bar\alpha_1)}
\end{equation}
where $\eta>0$ is a control constant that will be specified in each special case to achieve the best sample complexity, choose the stepsizes as 
\begin{equation}\label{eq.stepsize}
 \alpha_k=\min\left\{\bar\alpha_1, \bar\alpha_2, \frac{\alpha}{\sqrt{K}}\right\}~~~{\rm and}~~~\beta_k=\frac{8L_fL_y + \eta L_{yx}\tilde{C}_f^2\bar\alpha_1}{4T\mu_g}\alpha_k
\end{equation}
then for any $T\geq 1$, the iterates $\{x^k\}$ and $\{y^k\}$ generated by Algorithm \ref{alg: ALSET} satisfy
	\begin{equation}\label{eq.theorem1}
	\small
   \frac{1}{K} \sum_{k=1}^K\EE\left[\left\|\nabla F(x^k)\right\|^2\right] = {\cal O}\Big(\frac{1}{\sqrt{K}}\Big)~~~{\rm and}~~~\EE\left[\left\|y^K\!-y^*(x^K)\right\|^2\right] = {\cal O}\Big(\frac{1}{\sqrt{K}}\Big) 
\end{equation}
where $y^*(x^K)$ is the minimizer of the lower-level problem in \eqref{opt0-low}.
\end{customthm}

\begin{proposition}\label{prop.bilevel}
Under the same assumptions and the parameters of Theorem \ref{theorem1}, with $\kappa:=\frac{\ell_{g,1}}{\mu_g}$ being the condition number, if we select {\small$\alpha=\Theta(\kappa^{-2.5})$, $T=\Theta(\kappa^{4})$, $\eta={\cal O}(\kappa)$} in \eqref{eq.stepsize}, then we have 
\begin{equation}
    \frac{1}{K}\sum_{k=0}^{K-1}\EE[\|\nabla F(x^k)\|^2] = {\cal O}\left(\frac{\kappa^3}{K} + \frac{\kappa^{2.5}}{\sqrt{K}}\right).
\end{equation}
\end{proposition}
\textbf{Discussion of Theorem \ref{theorem1}.}
Theorem \ref{theorem1} implies that the convergence rate of ALSET to the stationary point of \eqref{opt0} is ${\cal O}(K^{-0.5})$.
Since each iteration of ALSET only uses $\widetilde{\cal O}(1)$ samples (see Algorithm \ref{alg: ALSET}), the sample complexity to achieve an $\epsilon$-stationary point of \eqref{opt0} is ${\cal O}(\epsilon^{-2})$, which is on the same order of SGD's sample complexity for the single-level nonconvex problems \citep{ghadimi2013sgd}, and improves the state-of-the-art single-loop TTSA's sample complexity ${\cal O}(\epsilon^{-2.5})$ \citep{hong2020ac}. 
Compared to \citep{ji2020provably}, ALSET achieves the same sample complexity both in terms of $\epsilon$ and $\kappa$,  without using an increasing batch size. Importantly, we obtain this tighter bound without introducing additional assumptions.

\subsection{Proof sketch}
In this subsection, we highlight the key steps of the proof towards Theorem \ref{theorem1}, and highlight the differences between our analysis and the existing ones. 

For simplicity of the convergence analysis, we define the following Lyapunov function as $\mathbb{V}^k \!:=\! F(x^k)  + \frac{L_f}{L_y}\|y^k - y^*(x^k)\|^2$.
We first quantify the difference between two Lyapunov functions as
\begin{align*}\label{eq.diff-Lya}
  \mathbb{V}^{k+1}  -  \mathbb{V}^k   =  & \underbracket{ F(x^{k+1}) -  F(x^k)}_{\rm Lemma~\ref{lemma3}}  
 ~ + ~ \frac{L_f}{L_y}(\underbracket{\|y^{k+1} - y^*(x^{k+1})\|^2 - \|y^k - y^*(x^k)\|^2 }_{\rm Lemma~\ref{lemma2}}).\numberthis
\end{align*}

The difference in \eqref{eq.diff-Lya} consists of two difference terms: the first term quantifies the descent of the overall objective functions; the second term characterizes the descent of the lower-level errors.

 We will first analyze the descent of the upper-level objective in the next lemma. 
\begin{lemma}[Descent of upper level]
\label{lemma3}
Suppose Assumptions 1--3 hold. 
Define $\bar{h}_f^k:=\EE[h_f^k|x^k, y^{k+1}]$ and $\|\bar{h}_f^k-\overline{\nabla}f(x^k,y^{k+1})\|\leq b_k$. The sequence of $x^k$ generated by Algorithm \ref{alg: ALSET} satisfies 
 \begin{align*}\label{eq.lemma3}
 \EE[F(x^{k+1})]-  \EE[F(x^k)] 
\leq&- \frac{\alpha_k}{2}\EE[\|\nabla F(x^k)\|^2] - \left(\frac{\alpha_k}{2}-\frac{L_F\alpha_k^2}{2}\right)\EE[\|\bar{h}_f^k\|^2]\\
    &+ L_f^2\alpha_k\EE[\|y^{k+1}-y^*(x^k)\|^2] + \alpha_kb_k^2 + \frac{L_F\alpha_k^2}{2}\tilde\sigma_f^2\numberthis
\end{align*} 
where constants $L_f, L_F, \sigma_f^2$ are defined in Lemma \ref{lemma_lip} of the supplementary document. 
\end{lemma}

Lemma \ref{lemma3} implies that the descent of the upper-level objective functions depends on the error of the lower-level variable $y^k$. 
  We will next analyze the error of the lower-level variable, which is the key step to improving the existing results. 


Before we analyze the error of $y^k$, we introduce a lemma that characterizes the smoothness of $y^*(x)$ and the bounded moments of $h_f^k$. The smoothness and the bounded moments have not been explored by previous analysis such as \cite{ghadimi2018bilevel,ji2020provably,hong2020ac}, and they play an essential role in our improved analysis of $y^k$.
\begin{lemma}[Smoothness and boundedness]\label{lemmasm}
Under Assumptions 1 and 2, we have 
	\begin{equation}
   \|\nabla y^*(x_1)-\nabla y^*(x_2)\|\leq  L_{yx}\|x_1-x_2\|;~~~\EE[\|h_f^k\|^2|x^k, y^{k+1}]\leq  \tilde{C}_f^2 
\end{equation}
where $L_{yx}$ and $\tilde{C}_f^2$ depend on the constants defined in Assumptions 1-2.
\end{lemma}

Building upon Lemma \ref{lemmasm}, we establish the progress of the lower-level update.
\begin{lemma}[Error of lower level]
\label{lemma2}
Suppose that Assumptions 1--3 hold, and $y^{k+1}$ is generated by running iteration \eqref{eq.ALSET} given $x^k$. If we choose $\beta_k\leq\frac{2}{\mu_g+\ell_{g,1}}$, then $y^{k+1}$ satisfies 
\begin{subequations}
\begin{align}
    &  \EE[\|y^{k+1}\!\!-y^*(x^k)\|^2]\leq(1-\mu_g\beta_k)^T\EE[\|y^k -y^*(x^k)\|^2] + T\beta_k^2\sigma_{g,1}^2\label{eq.lemma2-2} \\ 
  &   \EE[\|y^{k+1}-y^*(x^{k+1})\|^2]\leq \Big(1 + L_fL_y\alpha_k + \frac{\eta L_{yx}\tilde{C}_f^2}{4}\alpha_k^2\Big) \EE[\|y^{k+1}\!\!-y^*(x^k)\|^2]\nonumber\\
    &\qquad\qquad\qquad\qquad\qquad+ \Big(L_y^2 + \frac{L_y}{4L_f\alpha_k} + \frac{L_{yx}}{4\eta}\Big)\alpha_k^2\EE[\|\bar{h}_f^k\|^2] + \Big(L_y^2+\frac{L_{yx}}{4\eta}\Big)\alpha_k^2\tilde\sigma_f^2    \label{eq.lemma2-1}
\end{align}
\end{subequations}
where $\eta>0$ is a fixed constant that will be chosen to obtain the tighter complexity bound. 
\end{lemma}
Plugging \eqref{eq.lemma2-2} into \eqref{eq.lemma2-1}, and selecting stepsizes $\alpha_k, \beta_k$ properly, we can show that 
\begin{equation}\label{eq.lemma2-3}
    \EE[\|y^{k+1}-y^*(x^{k+1})\|^2] \leq (1-\delta_1)\EE[\|y^k -y^*(x^k)\|^2]+\delta_2\EE[\|\bar{h}_f^k\|^2] + \delta_3T\sigma_{g,1}^2+\delta_4\tilde\sigma_f^2  
\end{equation}
where the constants are $\delta_1\in [0,1), \delta_2={\cal O}(\alpha_k), \delta_3={\cal O}(\beta_k^2), \delta_4={\cal O}(\alpha_k^2)$. In our tighter analysis, the term $\EE[\|\bar{h}_f^k\|^2]$ will be canceled, so choosing $\alpha_k={\cal O}(k^{-0.5})$ and $\beta_k={\cal O}(k^{-0.5})$ makes the variance terms decrease at the same order as SGD for stochastic non-nested problems. 

As a comparison, the progress of the lower-level problem in \cite{hong2020ac,ji2020provably} can be summarized as 
\begin{equation}\label{eq.lemma2-4}
    \EE[\|y^{k+1}-y^*(x^{k+1})\|^2]\leq (1-\delta_1)\EE[\|y^k -y^*(x^k)\|^2]+ \delta_5 \sigma^2   
\end{equation}
where $\sigma^2$ is some variance term, and the constant is $\delta_5={\cal O}(\beta_k^2+\alpha_k^2/\beta_k)$ or ${\cal O}(1/B_k)$ with $B_k$ being the batch size at iteration $k$. 
To balance the two terms in $\delta_5={\cal O}(\beta_k^2+\alpha_k^2/\beta_k)$, two timescales of stepsizes are needed; and to reduce $\delta_5={\cal O}(1/B_k)$, a growing batch size $B_k={\cal O}(k)$ is needed.

\vspace{-0.2cm}
\section{Applications to Stochastic Min-Max and Compositional Problems}\label{sec.special}
\vspace{-0.2cm}
Building upon the general results for the bilevel problems in Section \ref{sec.scg}, this section will identify special features of the stochastic min-max and stochastic compositional problems, and customize the general results to yield state-of-the-art convergence results for two special nested problems.

\subsection{Stochastic min-max problems}
We first apply our results to the stochastic min-max problem \eqref{opt0-1}.
In this special case, the lower-level function is $g(x,y;\phi)=-f(x,y;\xi)$, and the bilevel gradient in \eqref{grad-deter-2} reduces to
\begin{equation}\label{grad-min-max}
	\nabla F(x):=\nabla_xf\big(x, y^*(x)\big)+\nabla_xy^*(x)^{\top}\nabla_y f\big(x, y^*(x)\big)=\nabla_xf\big(x, y^*(x)\big) 
\end{equation}
where the second equality follows from the optimality condition of the lower-level problem, i.e., $\nabla_yf(x,y^*(x))=0$. 
Similar to Section \ref{sec.scg}, we again approximate $\nabla F(x)$ on a certain vector $y$ in place of $y^*(x)$.
\begin{wrapfigure}{R}{0.57\textwidth}
\vspace{-0.2cm}
 \begin{minipage}{0.57\textwidth}
 \small
\begin{algorithm}[H]
\caption{ALSET for min-max problems}\label{alg: ALSET.min-max}
    \begin{algorithmic}[1]
    \State{\textbf{initialize:} $x^0, y^0$,  stepsizes $\{\alpha_k, \beta_k\}$.} 
    \For{$k=0,1,\ldots, K-1$}
        \State{set $y^{k,0}=y^k$}
        \For{$t=0,1,\ldots,T-1$}
            \State{update $y^{k,t+1}=y^{k,t}-\beta_k\nabla_yf(x^k,y^{k,t};\xi_1^k)$}
        \EndFor
        \State{set $y^{k+1}=y^{k,T}$}
        \State{update $x^{k+1}=x^k-\alpha_k\nabla_xf(x^k,y^{k+1};\xi_2^k)$}
    \EndFor
    \end{algorithmic}
\end{algorithm}
 \end{minipage}
 \vspace{-0.4cm}
\end{wrapfigure}
Therefore, the alternating stochastic gradients for this special case are given by
	\begin{equation}
    h_g^{k,t} =-\nabla_yf(x^k,y^{k,t};\xi_1^k)~~~{\rm and}~~~    h_f^k = \nabla_xf(x^k,y^{k+1};\xi_2^k).
\end{equation}
Plugging the stochastic gradient into the general update \eqref{eq.ALSET}, we summarize the update in Algorithm \ref{alg: ALSET.min-max}.
When the number of $y$-update is $T=1$, the ALSET algorithm reduces to the SGDA method in \cite{lin2020gradient}.  

\begin{proposition}\label{prop.min-max}
Under the same assumptions and the choice of parameters as those in Theorem \ref{theorem1}, if we select $\alpha=\Theta(\kappa^{-1})$, $T=\Theta(\kappa)$, $\eta=1$ in \eqref{eq.stepsize}, we have  
\begin{equation}
   \frac{1}{K}\sum_{k=0}^{K-1}\EE[\|\nabla F(x^k)\|^2]={\cal O}\left(\frac{\kappa^2}{K} + \frac{\kappa}{\sqrt{K}}\right).
\end{equation}
\end{proposition}

Proposition \ref{prop.min-max} implies that for the min-max problem, the convergence rate of ALSET to the stationary point of $F(x):=\max_{y\in\mathbb{R}^{d'}}~\mathbb{E}_{\xi}\left[f(x,y;\xi)\right]$ is ${\cal O}(K^{-0.5})$.
Since each iteration of ALSET only uses ${\cal O}(1)$ samples (see Algorithm \ref{alg: ALSET.min-max}), the sample complexity to achieve an $\epsilon$-stationary point of \eqref{opt0-1} is ${\cal O}(\epsilon^{-2})$. 
Comparing with the results in \cite{lin2020gradient}, we achieve the same sample complexity without an increasing batch size ${\cal O}(\epsilon^{-1})$, and improve their sample complexity ${\cal O}(\epsilon^{-2.5})$ under a fixed batch size. However, it is also worth mentioning that compared with \cite{lin2020gradient}, our analysis requires the additional Lipschitz continuity assumption of $f(x,y)$, which inherits from the analysis for the general bilevel problem. 
 

\subsection{Stochastic compositional problems}\label{sec.comp}
In this section, we apply our results to the stochastic compositional problem \eqref{opt0-2}.
In this special case, the upper-level function is $f(x,y;\xi):=f(y;\xi)$, and the lower-level function is $g(x,y;\phi)=\|y-h(x;\phi)\|^2$, and the bilevel gradient in \eqref{grad-deter-2} reduces to
\begin{align}\label{grad-comp}
	\nabla F(x):&=\nabla_xf\big(x, y^*(x)\big)-\nabla_{xy}^2g(x, y^*(x))\!\left[\nabla_{yy}^2 g(x, y^*(x))\right]^{-1}\nabla_y f(x, y^*(x))\nonumber\\
	&=\nabla h(x;\phi)^{\top}\nabla_y f(y^*(x))
\end{align}
where we use the fact that $\nabla_{yy}g(x,y;\phi)=\mathbf{I}_{d'\times d'}, \nabla_{xy}g(x,y;\phi)=-\nabla h(x;\phi)^{\top}$. 
Similar to Section \ref{sec.scg}, we again evaluate $\nabla F(x)$ on a certain vector $y$ in place of $y^*(x)$.
Therefore, the alternating stochastic gradients $h_f^k, h_g^{k,t}$ for this special case are much simpler in this case, given by
\begin{align}
    h_g^{k,t}=y^{k,t}-h(x^k;\phi^{k,t})~~~{\rm and}~~~ 
    h_f^k=\nabla h(x^k;\phi^k)\nabla f(y^{k+1};\xi^k).
\end{align}
Plugging the stochastic gradient into the general update \eqref{eq.ALSET}, we summarize the update in Algorithm \ref{alg: ALSET.comp}.
When $T=1$, the ALSET algorithm reduces to SCGD proposed in \cite{wang2017mp}. 

\begin{wrapfigure}{R}{0.58\textwidth}
\vspace{-0.6cm}
 \begin{minipage}{0.58\textwidth}
 \small
\begin{algorithm}[H]
\caption{ALSET for compositional problems}\label{alg: ALSET.comp}
    \begin{algorithmic}[1]
    \State{\textbf{initialize:} $x^0, y^0$,  stepsizes $\{\alpha_k, \beta_k\}$.} 
    \For{$k=0,1,\ldots, K-1$}
            \State{update $y^{k+1}=y^k-\beta_k(y^k - g(x^k;\phi^k))$}
        \State{update $x^{k+1}=x^k-\alpha_k\nabla f(y^{k+1};\xi^k)\nabla g(x^k;\phi^k)$}
    \EndFor
    \end{algorithmic}
\end{algorithm}
 \end{minipage}
 \vspace{-0.6cm}
 \end{wrapfigure}

In the supplementary document, we have verified that the standard assumptions of stochastic compositional optimization in \cite{wang2017mp,wang2017jmlr,ghadimi2020jopt,zhang2019nips,chen2020scsc} are sufficient for Assumptions 1--3 to hold. 
 
\begin{proposition}\label{prop.comp}
Under the same assumptions and the choice of parameters as those in Theorem \ref{theorem1}, select $T=1,\alpha=1, \eta=\frac{1}{L_{yx}}$ in \eqref{eq.stepsize}, and then it holds
	\begin{equation}\label{eq.theorem3}
   \frac{1}{K} \sum_{k=1}^K\EE\left[\left\|\nabla F(x^k)\right\|^2\right] = {\cal O}\Big(\frac{1}{\sqrt{K}}\Big).
\end{equation}
\end{proposition}

Since each iteration of ALSET only uses ${\cal O}(1)$ samples (see Algorithm \ref{alg: ALSET.comp}), Proposition \ref{prop.comp} implies that the sample complexity to achieve an $\epsilon$-stationary point of \eqref{opt0-2} is ${\cal O}(\epsilon^{-2})$. 
Comparing with the results of the SCGD method in \cite{wang2017mp}, our result improves the sample complexity ${\cal O}(\epsilon^{-4})$ under a fixed batch size. Importantly, our analysis does not introduce additional assumption compared to \cite{wang2017mp}.

\section{Applications to Single-Timescale Actor-Critic Method}\label{sec.ac}
In this section, we apply our tighter analysis to the actor-critic (AC) method with linear value function approximation \cite{konda1999actor}, which can be viewed as a special case of the stochastic bilevel algorithm. 

Consider a Markov decision process described by $\mathcal{M}=\{ \mathcal{S}, \mathcal{A}, \mathcal{P}, R, \gamma \}$, where $\mathcal{S}$ is the state space, $\mathcal{A}$ is the action space, $\mathcal{P}(s'|s,a)$ is the probability of transitioning to $s'\in \mathcal{S}$ given state $s \in \mathcal{S}$ and action $a \in \mathcal{A}$, and $R(s,a,s')$ is the reward associated with $(s,a,s')$, and $\gamma \in [0,1)$ is a discount factor. 
For a policy $\pi_\theta$, define the value function $V_{\pi_\theta}(s)$ that satisfies the Bellman equation \citep{SuttonRL} 
\begin{equation}
V_{\pi_\theta}(s)=   \mathbb{E}_{ a \sim \pi_\theta(\cdot|s),\, s' \sim \mathcal{P}(\cdot|s,a)} \left[ r(s,a,s') + \gamma V_{\pi_\theta}(s')\right].
\end{equation}
Given the state feature mapping $\phi(\cdot):\mathcal{S}\xrightarrow[]{}\mathbb{R}^{d_y}$, we approximate the value function linearly as $V_{\pi_\theta}(s) \approx \hat{V}_y(s):=\phi(s)^\top y$, where $y \in \mathbb{R}^{d_y}$ is the critic parameter.   
The task of finding the best $y$ such that $V_{\pi_\theta}(s) \approx \hat{V}_y(s)$ is usually addressed by TD learning \citep{sutton1988td}. 

Defining the stationary distribution induced by the policy parameter $\theta_k$ as $\mu_{\theta_k}$ and the $k$th transition as $\xi_k \coloneqq (s_k,a_k,s_{k+1})$, which is sampled from $s_k\sim \mu_{\theta_k}, a \sim \pi_{\theta_k}, s_{k+1}\sim \mathcal{P}$, the TD-error is
\begin{equation}
    \hat{\delta}(\xi_k,y_k) \coloneqq r(s_k,a_k,s_{k+1}) + \gamma \phi(s_{k+1})^\top y_k - \phi(s_k)^\top y_k
\end{equation}
and the critic gradient $h_g(\xi_k,y_k) := \hat{\delta}(\xi_k,y_k) \nabla \hat{V}_{y_k}(s_k)$.
We update the parameter $y$ via
\begin{equation}\label{eq:td-update}
    y_{k+1} = \Pi_{R_y}\big(y_k + \beta_k  h_g(\xi_k,y_k)\big),
\end{equation}
where $\beta_k$ is the critic stepsize, and $\Pi_{R_y}$ is the projection to control the norm of the gradient. A pre-defined constant $R_y$ will be specified in the supplementary document. 

The goal of policy optimization is to solve $\max_{\theta \in \mathbb{R}^{d}} F(\theta)$ with $F(\theta) \coloneqq \E_{s \sim \eta} [V_{\pi_\theta}(s)]$, where $\eta$ is the initial distribution. 
Leveraging the value function approximation and the policy gradient theorem \cite{SuttonPG}, we have the policy gradient $ h_f(\xi,\theta, y):= \hat{\delta}(\xi,y)\psi_\theta(s,a)$, 
which gives the policy update 
\begin{equation}\label{eq:pg-update}
    \theta_{k+1}  = \theta_k + \alpha_k h_f(\xi_k',\theta_k, y_{k+1}),
\end{equation}
where $\alpha_k$ is the stepsize and $\psi_\theta(s,a):=\nabla \log\pi_\theta(a|s)$. Note that the sample $\xi_k' \coloneqq (s_k',a_k',s_{k+1}')$ used in \eqref{eq:pg-update} is independent from $\xi_k$ in \eqref{eq:td-update}. Specifically, $\xi_k'$ is sampled from $s_k'\sim d_{\theta_k}, a_k' \sim \pi_{\theta_k}, s_{k+1}'\sim \mathcal{P}$ with $d_{\theta_k}$ being the discounted state action visitation measure under $\theta_k$.

The alternating AC update \eqref{eq:td-update}-\eqref{eq:pg-update} is a special case of ALSET, where the critic update is the lower-level update, and the actor update is the upper-level update.

	
Due to space limitation, we will directly present the results of the alternating AC next, and defer presentation of the proof and the corresponding assumptions, which are the counterparts of Assumptions 1--3 in the context of AC, to the supplementary document. 

\begin{customthm}{2}[Actor-critic]\label{theorem2}
Under the some regularity conditions that are specified in the supplementary document, selecting step size $\alpha_k = \alpha={\cal O}(\frac{1}{\sqrt{K}})$, $\beta_k = \beta={\cal O}(\frac{1}{\sqrt{K}})$, it holds 
 \begin{equation}\label{eq:theorem:async-tts-actor-iid}
 \frac{1}{K}\sum_{k=1}^{K}\E\|\nabla F(\theta_k)\|_2^2 \!=\!
    \mathcal{O}\left(\frac{1}{\sqrt{K}}\right)+ \epsilon_{\mathrm{app}}
\end{equation}
where $\epsilon_{\mathrm{app}}$, defined in the supplementary document, captures the richness of the linear function class.
\end{customthm}

\textbf{Both sides of Theorem \ref{theorem2}.}
As an application of our tighter analysis, Theorem \ref{theorem2} establishes for the first time that the sample complexity of the single-loop alternating actor-critic method is ${\cal O}(\epsilon^{-2})$. 
On the positive side, this new result improves the previous complexity ${\cal O}(\epsilon^{-2.5})$ for the single-loop AC \cite{wu2020finite}, and ${\cal O}(\epsilon^{-2}\log\epsilon^{-1})$ for the nested-loop AC \cite{xu2020improving}, and matches ${\cal O}(\epsilon^{-2})$ for a recently developed AC with an exact critic oracle \cite{fu2020single}. In addition to using two independent samples, one limitation of our result is that inheriting from the analysis for the general bilevel case, our analysis of AC requires the smoothness of the critic fixed-point $y^*(\theta)$. As shown in the supplementary document, this implicitly requires the additional smooth assumption on the stationary distribution $\mu_{\theta}$. The removal of this  assumption and the extension to Markovian sampling are left for future research.

\appendix

 

\section{Proof for stochastic bilevel problem}

\subsection{Auxiliary Lemmas}
Throughout the proof, we use ${\cal F}_{k,t}=\sigma\{y^0, x^0, \ldots, y^k, x^k, y^{k,1},\ldots, y^{k,t}\}$, ${\cal F}_k'=\sigma\{y^0, x^0,\ldots, y^{k+1}\}$, 
where $\sigma\{\cdot\}$ denotes the $\sigma$-algebra generated by the random variables.

We first present some results that will be used frequently in the proof. 
\begin{proposition}[Restatement of Proposition \ref{prop1}]
Under Assumptions 1--3, we have the gradients
	\begin{align}
 \nabla F(x)=\nabla_xf(x, y^*(x))-\nabla_{xy}^2g(x, y^*(x))\!\left[\nabla_{yy}^2 g(x, y^*(x))\right]^{-1}\nabla_y f(x, y^*(x)).
\end{align}
\end{proposition}

\textit{Proof.}
Define the Jacobian matrix
\begin{align*}
  \nabla_xy(x) = 
  \begin{bmatrix}
    \frac{\partial}{\partial x_1} y_1 (x) & \cdots & \frac{\partial}{\partial x_d} y_1 (x)\\
    & \cdots & \\
    \frac{\partial}{\partial x_1} y_{d'}(x) & \cdots & \frac{\partial}{\partial x_d} y_{d'} (x)
  \end{bmatrix}.
\end{align*}
By the chain rule, it follows that
\begin{align}\label{grad-deter}
	\nabla F(x):=\nabla_xf\big(x, y^*(x)\big)+\nabla_xy^*(x)^{\top}\nabla_y f\big(x, y^*(x)\big). 
\end{align}
The minimizer $y^*(x)$ satisfies 
\begin{align}
    \nabla_y g(x,y^*(x)) = 0,\quad\text{thus}~~~\nabla_x\big(\nabla_y g(x,y^*(x))\big) = 0,
\end{align}
from which and the chain rule, it follows that
\begin{align*}
    \nabla_{xy}^2g\big(x, y^*(x)\big)
    + \nabla_xy^*(x)^{\top}\nabla_{yy}^2 g\big(x, y^*(x)\big) = 0.
\end{align*}
By Assumption 2,  $\nabla_{yy}^2 g\big(x, y^*(x)\big)$ is invertible, so from the last equation,
\begin{align}\label{grad-deter-1}
	\nabla_xy^*(x)^{\top}:=-\nabla_{xy}^2g\big(x, y^*(x)\big)\left[\nabla_{yy}^2 g\big(x, y^*(x)\big)\right]^{-1}. 
\end{align}
Substituting \eqref{grad-deter-1} into \eqref{grad-deter} yields \eqref{grad-deter-2}.


\begin{lemma}[{\citep[Lemma 2.2]{ghadimi2018bilevel}}]\label{lemma_lip} 
Under Assumptions 1 and 2, we have
\begin{subequations}
		\begin{align}
	\|\overline{\nabla}_x f(x, y^*(x))-\overline{\nabla}_x f(x, y)\|
&	\leq  L_f\|y^*(x)-y\|\\
	\|\nabla F(x_1)-\nabla F(x_2)\|
&	\leq  L_F\|x_1-x_2\|\\
	\|y^*(x_1)-y^*(x_2)\|&\leq  L_y\|x_1-x_2\|
\end{align}
\end{subequations}
with the constants $L_f, L_y, L_F$ given by  
\begin{align*}
L_{f}&:=\ell_{f,1} + \frac{\ell_{g,1}\ell_{f,1}}{\mu_g} + \frac{\ell_{f,0}}{\mu_g}\Big(\ell_{g,2}+\frac{\ell_{g,1}{\ell_{g,2}}}{\mu_g}\Big)={\cal O}(\kappa^2),~~~
L_y :=\frac{\ell_{g,1}}{\mu_g}={\cal O}(\kappa)\\
L_{F}&:=\ell_{f,1} + \frac{\ell_{g,1}(\ell_{f,1}+L_f)}{\mu_g} + \frac{\ell_{f,0}}{\mu_g}\Big(\ell_{g,2}+\frac{\ell_{g,1}{\ell_{g,2}}}{\mu_g}\Big)={\cal O}(\kappa^3),
\end{align*}
where the other constants are defined in Assumptions 1--3.
\end{lemma}

\begin{lemma}[{\citep[Lemma 11]{hong2020ac}}]\label{lemm_bias_var}
Recall the definition of $h_f^k$ in \eqref{eq.biased-estimator}. Define $$\bar{h}_f^k:=\EE[h_f^k|{\cal F}_k'].$$
We have
\begin{align*}
    &\|\overline{\nabla}_xf(x^k,y^{k+1})-\bar{h}_f^k\|\leq \ell_{g,1}\ell_{f,1}\frac{1}{\mu_g}\left(1-\frac{\mu_g}{\ell_{g,1}}\right)^N=:b_k\\
    &\EE[\|h_f^k-\bar h_f^k\|^2]\leq\sigma_f^2 + \frac{3}{\mu_g^2}\left[(\sigma_f^2+\ell_{f,0}^2)(\sigma_{g,2}^2+2\ell_{g,1}^2)+\sigma_f^2\ell_{g,1}^2\right]=:\tilde\sigma_f^2={\cal O}(\kappa^2),
\end{align*}
where $\kappa$ is the condition number defined below Assumption 2.
\end{lemma}

\subsection{Proof of Lemma \ref{lemma3}}
Using the Lipschitz property of $\nabla F$ in Lemma \ref{lemma_lip}, we have
\begin{align*}\label{eq.lemma3-pf-1}
    \EE[F(x^{k+1})|{\cal F}_k']
    &\leq F(x^k) + \EE[\dotp{\nabla F(x^k), x^{k+1}-x^k}|{\cal F}_k'] + \frac{L_F}{2}\EE[\|x^{k+1}-x^k\|^2|{\cal F}_k']\\
    &= F(x^k) - \alpha_k\dotp{\nabla F(x^k),\bar{h}_f^k} + \frac{L_F\alpha_k^2}{2}\EE[\|h_f^k\|^2|{\cal F}_k']\\
    &\stackrel{(a)}{=} F(x^k) - \frac{\alpha_k}{2}\|\nabla F(x^k)\|^2 - \frac{\alpha_k}{2}\|\bar{h}_f^k\|^2 +  \frac{\alpha_k}{2}\|\nabla F(x^k)-\bar{h}_f^k\|^2 \\
    &\quad\quad\quad\quad + \frac{L_F\alpha_k^2}{2}\|\bar{h}_f^k\|^2 + \frac{L_F\alpha_k^2}{2}\EE[\|h_f^k-\bar{h}_f^k\|^2|{\cal F}_k']\\
    &\stackrel{(b)}{\leq} F(x^k) - \frac{\alpha_k}{2}\|\nabla F(x^k)\|^2 - \left(\frac{\alpha_k}{2}-\frac{L_F\alpha_k^2}{2}\right)\|\bar{h}_f^k\|^2\\
    &\quad\quad\quad\quad +\frac{\alpha_k}{2}\|\nabla F(x^k)-\bar{h}_f^k\|^2 + \frac{L_F\alpha_k^2}{2}\tilde\sigma_f^2\numberthis
\end{align*}
where (a) uses $2a^{\top}b=\|a\|^2+\|b\|^2-\|a-b\|^2$ twice and (b) uses Lemma \ref{lemm_bias_var}.

We decompose the gradient bias term as follows
\begin{align*}\label{eq.lemma3-pf-2}
   \|\nabla F(x^k)-\bar{h}_f^k\|^2
   &= \|\overline{\nabla}f(x^k,y^*(x^k)) - \overline{\nabla}f(x^k,y^{k+1}) + \overline{\nabla}f(x^k,y^{k+1})-\bar{h}_f^k\|^2\\
   &\leq 2\|\overline{\nabla}f(x^k,y^*(x^k)) - \overline{\nabla}f(x^k,y^{k+1})\|^2 + 2\|\overline{\nabla}f(x^k,y^{k+1})-\bar{h}_f^k\|^2\\
   &\stackrel{(a)}{\leq} 2L_f^2\big\|y^{k+1}-y^*(x^k)\big\|^2 + 2b_k^2\numberthis
\end{align*}
where (a) follows from Lemma \ref{lemma_lip} and Lemma \ref{lemm_bias_var}. 
Plugging \eqref{eq.lemma3-pf-2} into \eqref{eq.lemma3-pf-1} completes the proof.

\subsection{Proof of Lemma \ref{lemmasm}}
Recalling the definition of $\nabla_xy^*(x)$ in \eqref{grad-deter-1}, for any $x_1, x_2$, we have
\begin{align*}
    &~\|\nabla_xy^*(x_1)-\nabla_xy^*(x_2)\|\numberthis\\
    =&~\|\nabla_{xy}^2g(x_1, y^*(x_1))[\nabla_{yy}^2 g(x_1, y^*(x_1))]^{-1}-\nabla_{xy}^2g(x_2, y^*(x_2))[\nabla_{yy}^2 g(x_2, y^*(x_2))]^{-1}\|\\
    \leq&~\|\nabla_{xy}^2g(x_1, y^*(x_1))-\nabla_{xy}^2g(x_2, y^*(x_2))\|\|[\nabla_{yy}^2 g(x_1, y^*(x_1))]^{-1}\|\\
        &~+\|\nabla_{xy}^2g(x_2, y^*(x_2))\| \|[\nabla_{yy}^2 g(x_1, y^*(x_1))]^{-1}-[\nabla_{yy}^2 g(x_2, y^*(x_2))]^{-1}\|\\
  \stackrel{(a)}{\leq} &        \frac{1}{\mu_g}\|\nabla_{xy}^2g(x_1, y^*(x_1))-\nabla_{xy}^2g(x_2, y^*(x_2))\|\\
        &~+\ell_{g,1}\|[\nabla_{yy}^2 g(x_1, y^*(x_1))]^{-1}\left(\nabla_{yy}^2 g(x_1, y^*(x_1))-\nabla_{yy}^2 g(x_2, y^*(x_2))\right)[\nabla_{yy}^2 g(x_2, y^*(x_2))]^{-1}\|\\
  \stackrel{(b)}{\leq}&~\frac{1}{\mu_g}\|\nabla_{xy}^2g(x_1, y^*(x_1))-\nabla_{xy}^2g(x_2, y^*(x_2))\| + \frac{\ell_{g,1}}{\mu_g^2}\|\nabla_{yy}^2 g(x_1, y^*(x_1))-\nabla_{yy}^2 g(x_2, y^*(x_2))\|
\end{align*}
where both (a) and (b) follow from Assumption 1 and 2.

In addition, we have that
\begin{align*}
    &~\frac{1}{\mu_g}\|\nabla_{xy}^2g(x_1, y^*(x_1))-\nabla_{xy}^2g(x_2, y^*(x_2))\| + \frac{\ell_{g,1}}{\mu_g^2}\|\nabla_{yy}^2 g(x_1, y^*(x_1))-\nabla_{yy}^2 g(x_2, y^*(x_2))\|\\
    \leq&~\frac{\ell_{g,2}}{\mu_g}\|x_1-x_2\| +\frac{\ell_{g,2}}{\mu_g}\| y^*(x_1)- y^*(x_2)\|+ \frac{\ell_{g,1}\ell_{g,2}}{\mu_g^2}\|x_1-x_2\|+\frac{\ell_{g,1}\ell_{g,2}}{\mu_g^2}\| y^*(x_1)- y^*(x_2)\|\\    
   \stackrel{(c)}{\leq}&~\left(\frac{\ell_{g,2}+\ell_{g,2}L_y}{\mu_g}+ \frac{\ell_{g,1}(\ell_{g,2}+\ell_{g,2}L_y)}{\mu_g^2}\right)\|x_1-x_2\|\numberthis
\end{align*}
where (c) follows from Lemma \ref{lemma_lip}.

Next we derive the bound of $h_f^k$,
\begin{align*}
    \EE[\|h_f^k\|^2|{\cal F}_k']
    &=\|\bar{h}_f^k\|^2+ \EE[\|h_f^k-\bar{h}_f^k\|^2|{\cal F}_k']\\
    &   \stackrel{(d)}{\leq} (\|\overline{\nabla}f(x^k,y^{k+1})\| + \|\bar{h}_f^k-\overline{\nabla}f(x^k,y^{k+1})\|)^2 + \tilde\sigma_f^2\\
    &   \stackrel{(e)}{\leq} \left(\ell_{f,0} + \frac{\ell_{f,0}\ell_{g, 1}}{\mu_g} + \frac{\ell_{g,1}\ell_{f,1}}{\mu_g}\left(1-\frac{\mu_g}{\ell_{g,1}}\right)^N\right)^2 + \tilde\sigma_f^2 \\
    &\leq \left(\ell_{f,0} + \frac{\ell_{f,0}\ell_{g, 1}}{\mu_g} + \frac{\ell_{g,1}\ell_{f,1}}{\mu_g}\right)^2 + \tilde\sigma_f^2 \numberthis
\end{align*}
where (d) is from Lemma \ref{lemm_bias_var}, and (e) is due to 
\begin{align*}
    \|\overline{\nabla}_xf\big(x, y\big)\|
    &=\|\nabla_xf\big(x, y\big) -\nabla_{xy}^2g\big(x, y\big)\left[\nabla_{yy}^2 g\big(x, y\big)\right]^{-1}\nabla_y f\big(x, y\big)\|\\
    &\leq \|\nabla_xf\big(x, y\big)\| + \|\nabla_{xy}^2g\big(x, y\big)\|\left\|\left[\nabla_{yy}^2 g\big(x, y\big)\right]^{-1}\right\|\|\nabla_y f\big(x, y\big)\|\\
    &\leq \ell_{f,0} + \ell_{g,1}\frac{1}{\mu_g}\ell_{f,0}.
\end{align*}

As a result, we have
\begin{align}
    &L_{yx}:=\frac{\ell_{g,2}+\ell_{g,2}L_y}{\mu_g} + \frac{\ell_{g,1}(\ell_{g,2}+\ell_{g,2}L_y)}{\mu_g^2}={\cal O}(\kappa^3)\\
    &\tilde{C}_f^2: = \left(l_{f,0} + \frac{\ell_{g,1}}{\mu_g}\ell_{f,1} + \ell_{g,1}\ell_{f,1}\frac{1}{\mu_g}\right)^2 + \tilde\sigma_f^2 = {\cal O}(\kappa^2)
\end{align}
from which the proof is complete.

\subsection{Proof of Lemma \ref{lemma2}}
 
We start by decomposing the error of the lower level variable as
\begin{align*}\label{eq.lemma2-new-pf-1}
\|y^{k+1}-y^*(x^{k+1})\|^2
&=\underbracket{\|y^{k+1}-y^*(x^k)\|^2}_{J_1} + \underbracket{\|y^*(x^{k+1})-y^*(x^k)\|^2}_{J_2} \\
&\quad + \underbracket{2\dotp{y^{k+1}-y^*(x^k), y^*(x^k)-y^*(x^{k+1})}}_{J_3}.\numberthis
\end{align*}

Notice that $y^{k+1}=y^{k,T}$ as defined in \eqref{eq.ALSET3}. 
We first analyze
\begin{align*}\label{eq.pf.lm3-1}
    &~~~~~\EE[\|y^{k,t+1}-y^*(x^k)\|^2|{\cal F}_k^t]\\
    &= \EE[\|y^{k,t}-\beta_kh_g^{k,t}-y^*(x^k)\|^2|{\cal F}_k^t]\\
    &= \|y^{k,t}-y^*(x^k)\|^2 - 2\beta_k\dotp{y^{k,t}-y^*(x^k),\EE[h_g^{k,t}|{\cal F}_k^t]} + \beta_k^2\EE[\|h_g^{k,t}\|^2|{\cal F}_k^t]\\
    &\stackrel{(a)}{\leq} \|y^{k,t}-y^*(x^k)\|^2 - 2\beta_k\dotp{y^{k,t}-y^*(x^k),\nabla g(x^k, y^{k,t})} + \beta_k^2\|\nabla g(x^k, y^{k,t})\|^2 + \beta_k^2\sigma_{g,1}^2\\
    &\stackrel{(b)}{\leq} \Big(1-\frac{2\mu_g\ell_{g,1}}{\mu_g+\ell_{g,1}}\beta_k\Big)\|y^{k,t}-y^*(x^k)\|^2 + \beta_k\Big(\beta_k-\frac{2}{\mu_g+\ell_{g,1}}\Big)\|\nabla_yg(x^k,y^{k,t})\|^2 + \beta_k^2\sigma_{g,1}^2\\
    &\stackrel{(c)}{\leq} (1-\rho_g\beta_k)\|y^{k,t}-y^*(x^k)\|^2 + \beta_k^2\sigma_{g,1}^2 \numberthis
\end{align*}
where (a) comes from the fact that $\Var[X]=\EE[X^2]-\EE[X]^2$, (b) follows from the $\mu_g$-strong convexity and $\ell_{g,1}$ smoothness of $g(x, y)$ \citep[Theorem 2.1.11]{nesterov2013}, and (c) follows from the choice of stepsize $\beta_k\leq \frac{2}{\mu_g+\ell_{g,1}}$ in \eqref{eq.stepsize} and the definition of $\rho_g:=\frac{2\mu_g\ell_{g,1}}{\mu_g+\ell_{g,1}}$.

Taking expectation over ${\cal F}_k^t$ on both sides of \eqref{eq.pf.lm3-1} and using induction, we are able to get
\begin{align}\label{eq.lemma2-new-pf-2}
\EE[J_1]=\EE[\|y^{k+1}-y^*(x^k)\|^2]\leq(1-\rho_g\beta_k)^T\EE[\|y^k-y^*(x^k)\|^2] + T\beta_k^2\sigma_{g,1}^2.
\end{align}

The upper bound of $J_2$ can be derived as
\begin{align*}\label{eq.lemma2-new-pf-3}
\EE[J_2]
&=\EE[\|y^*(x^{k+1})-y^*(x^k)\|^2]\leq L_y^2\EE[\|x^{k+1}-x^k\|^2]\\
&=L_y^2\alpha_k^2\EE\left[\EE[\|h_f^k-\bar{h}_f^k + \bar{h}_f^k\|^2|{\cal F}_k']\right]\\
&\leq L_y^2\alpha_k^2(\EE[\|\bar{h}_f^k\|^2] + \tilde\sigma_f^2)\numberthis
\end{align*}
where the inequality follows from Lemma \ref{lemm_bias_var}.

The term $J_3$ can be decomposed as
\begin{align*}\label{eq.lemma2-new-pf-4-0}
    \EE[J_3]
    &= \underbracket{-\EE[\dotp{y^{k+1}\!\!-y^*(x^k), \nabla y^*(x^k)(x^{k+1}\!\!-x^k)}]}_{J_3^1}\\
    &\quad \underbracket{-\EE[\dotp{y^{k+1}\!\!-y^*(x^k), y^*(x^{k+1})-y^*(x^k)-\nabla y^*(x^k)(x^{k+1}\!\!-x^k)}]}_{J_3^2}. \numberthis
\end{align*}

Using the alternating update of $x$ and $y$, e.g., $x^k \rightarrow y^{k+1} \rightarrow x^{k+1}$, we can bound $J_3^1$ by
\begin{align*}\label{eq.lemma2-new-pf-4-1}
    -\EE[\dotp{y^{k+1}\!\!-y^*(x^k), \nabla y^*(x^k)(x^{k+1}\!\!-x^k)}]
    =&-\EE[\dotp{y^{k+1}\!\!-y^*(x^k), \EE[\nabla y^*(x^k)(x^{k+1}\!\!-x^k)\mid {\cal F}_k']}]\\
 \stackrel{(d)}{=}&-\alpha_k\EE[\dotp{y^{k+1}\!\!-y^*(x^k),   \nabla y^*(x^k)\bar{h}_f^k}]\\
 \leq & \alpha_k\EE[\|y^{k+1}\!\!-y^*(x^k)\|  \|\nabla y^*(x^k)\bar{h}_f^k\|]\\
 \stackrel{(e)}{\leq} & \alpha_kL_y\EE[\|y^{k+1}\!\!-y^*(x^k)\|\|\bar{h}_f^k\|]  \\
 \stackrel{(f)}{\leq} & \gamma_k\EE[\|y^{k+1}\!\!-y^*(x^k)\|^2] + \frac{L_y^2\alpha_k^2}{4\gamma_k}\EE[\|\bar{h}_f^k\|^2]\numberthis
\end{align*}
where (d) uses the fact that $\bar{h}_f^k=\EE[h_f^k|{\cal F}_k']$; (e) follows from Lemma \ref{lemma_lip}; and (f) uses the Young's inequality such that $ab\leq \gamma_k a^2+\frac{b^2}{4\gamma_k}$.

Using the smoothness of $y^*(x)$ in Lemma \ref{lemmasm}, we can bound $J_3^2$ by
\begin{align*}\label{eq.lemma2-new-pf-4-2}
   & -\EE[\dotp{y^{k+1}\!\!-y^*(x^k), y^*(x^{k+1})-y^*(x^k)-\nabla y^*(x^k)(x^{k+1}\!\!-x^k)}]\\
\leq & \EE[\|y^{k+1}\!\!-y^*(x^k)\| \|y^*(x^{k+1})-y^*(x^k)-\nabla y^*(x^k)(x^{k+1}\!\!-x^k)\|]\\
\leq & \frac{L_{yx}}{2}\EE\left[\|y^{k+1}\!\!-y^*(x^k)\|\|x^{k+1}\!\!-x^k\|^2\right]\\
\leq & \frac{\eta L_{yx}}{4}\EE\left[\|y^{k+1}\!\!-y^*(x^k)\|^2\EE[\|x^{k+1}\!\!-x^k\|^2|{\cal F}_k']\right]+ \frac{L_{yx}}{4\eta}\EE\left[\EE[\|x^{k+1}\!\!-x^k\|^2|{\cal F}_k']\right]\\
 \stackrel{(g)}{\leq} & \frac{\eta L_{yx}\tilde{C}_f^2\alpha_k^2}{4}\EE[\|y^{k+1}\!\!-y^*(x^k)\|^2] + \frac{L_{yx}\alpha_k^2}{4\eta}(\EE[\|\bar{h}_f^k\|^2] + \tilde\sigma_f^2)\numberthis
\end{align*}
where (g) uses the fact that $\EE[\|h_f^k\|^2|{\cal F}_k']\leq  \tilde{C}_f^2$ in Lemma \ref{lemmasm} and the variance bound in Lemma \ref{lemm_bias_var}.

Plugging \eqref{eq.lemma2-new-pf-4-1} and \eqref{eq.lemma2-new-pf-4-2} into \eqref{eq.lemma2-new-pf-4-0}, we have
\begin{align*}\label{eq.lemma2-new-pf-4}
    \EE[J_3]
    \leq\left(\gamma_k + \frac{\eta L_{yx}\tilde{C}_f^2}{4}\alpha_k^2\right)\EE[\|y^{k+1}\!\!-y^*(x^k)\|^2] + \left(\frac{L_y^2\alpha_k^2}{4\gamma_k}+\frac{L_{yx}\alpha_k^2}{4\eta}\right)\EE[\|\bar{h}_f^k\|^2] + \frac{L_{yx}\alpha_k^2}{4\eta}\tilde\sigma_f^2.\numberthis
\end{align*}

Plugging \eqref{eq.lemma2-new-pf-3}, \eqref{eq.lemma2-new-pf-4} into \eqref{eq.lemma2-new-pf-1}, we get
\begin{align*}
    \EE[\|y^{k+1}-y^*(x^{k+1})\|^2]&\leq \Big(1 + \gamma_k + \frac{\eta L_{yx}\tilde{C}_f^2}{4}\alpha_k^2\Big)\EE[\|y^{k+1}-y^*(x^k)\|^2]\\
    &\quad + \Big(L_y^2\alpha_k^2+ \frac{L_y^2\alpha_k^2}{4\gamma_k} + \frac{L_{yx}\alpha_k^2}{4\eta}\Big)\EE[\|\bar{h}_f^k\|^2] + \Big(L_y^2\alpha_k^2+\frac{L_{yx}\alpha_k^2}{4\eta}\Big)\tilde\sigma_f^2
\end{align*}
from which the proof is complete.

\subsection{Proof of Theorem \ref{theorem1}}

Using Lemma \ref{lemma3} and \ref{lemma2}, we, respectively, bound the two difference terms in \eqref{eq.diff-Lya} and obtain
\begin{align*}\label{eq.thm1-pf-1}
   \EE[\mathbb{V}^{k+1}]& - \EE[\mathbb{V}^k] 
  \\
   \leq    &- \frac{\alpha_k}{2}\EE[\|\nabla F(x^k)\|^2] - \Big(\frac{\alpha_k}{2}-\frac{L_F\alpha_k^2}{2}-\frac{L_f}{L_y} L_y^2\alpha_k^2- \frac{L_f}{L_y}\frac{\alpha_k^2L_y^2}{4\gamma_k}-\frac{L_f}{L_y}\frac{L_{yx}\alpha_k^2}{4\eta}\Big)\EE[\|\bar{h}_f^k\|^2]\\
    &+ \frac{L_f}{L_y}\Big(1 + \gamma_k + L_fL_y\alpha_k + \frac{\eta L_{yx}\tilde{C}_f^2}{4}\alpha_k^2 \Big)\EE[\|y^{k+1}-y^*(x^k)\|^2]-\frac{L_f}{L_y}\EE[\|y^k-y^*(x^k)\|^2]\\
    &+ \alpha_kb_k^2 + \Big(\frac{L_F}{2} + \frac{L_f}{L_y} L_y^2 + \frac{L_f}{L_y}\frac{L_{yx}}{4\eta}\Big)\alpha_k^2\tilde\sigma_f^2\\
    \leq &- \frac{\alpha_k}{2}\EE[\|\nabla F(x^k)\|^2] - \Big(\frac{\alpha_k}{2}-\frac{L_F\alpha_k^2}{2}-\frac{L_f}{L_y} L_y^2\alpha_k^2- \frac{L_f}{L_y} \frac{\alpha_k^2L_y^2}{4\gamma_k}-\frac{L_f}{L_y} \frac{L_{yx}\alpha_k^2}{4\eta}\Big)\EE[\|\bar{h}_f^k\|^2]\\
    &+ \frac{L_f}{L_y}\Big(\big(1 + \gamma_k + L_fL_y\alpha_k + \frac{\eta L_{yx}\tilde{C}_f^2}{4}\alpha_k^2 \big)(1-\rho_g\beta_k)^T -1\Big)\EE[\|y^k-y^*(x^k)\|^2] \\
    &+ \frac{L_f}{L_y}\Big(1 + \gamma_k + L_fL_y\alpha_k + \frac{\eta L_{yx}\tilde{C}_f^2}{4}\alpha_k^2 \Big)T\beta_k^2\sigma_{g,1}^2 + \alpha_kb_k^2 + \Big(\frac{L_F}{2} + \frac{L_f}{L_y} L_y^2 + \frac{L_f}{L_y}\frac{L_{yx}}{4\eta}\Big)\alpha_k^2\tilde\sigma_f^2.\numberthis
\end{align*}

Selecting $\gamma_k=L_fL_y\alpha_k$, we can simplify \eqref{eq.thm1-pf-1} as
\begin{align*}
    \EE[\mathbb{V}^{k+1}] - \EE[\mathbb{V}^k]\label{eq.thm1-pf-2}
    \leq &- \frac{\alpha_k}{2}\EE[\|\nabla F(x^k)\|^2] - \Big(\frac{\alpha_k}{2}-\frac{L_F\alpha_k^2}{2}- L_fL_y\alpha_k^2- \frac{\alpha_k}{4}-\frac{L_f}{L_y}\frac{L_{yx}\alpha_k^2}{4\eta}\Big)\EE[\|\bar{h}_f^k\|^2]\\
    &+ \frac{L_f}{L_y}\Big(\big(1 + 2L_fL_y\alpha_k + \frac{\eta L_{yx}\tilde{C}_f^2}{4}\alpha_k^2 \big)(1-\rho_g\beta_k)^T -1\Big)\EE[\|y^k-y^*(x^k)\|^2] \\
    &+ \frac{L_f}{L_y}\Big(1 + 2L_fL_y\alpha_k + \frac{\eta L_{yx}\tilde{C}_f^2}{4}\alpha_k^2 \Big)T\beta_k^2\sigma_{g,1}^2\\
    &+ \alpha_kb_k^2 + \Big(\frac{L_F}{2} + L_fL_y +\frac{L_{yx}L_f}{4\eta L_y}\Big)\alpha_k^2\tilde\sigma_f^2.\numberthis
\end{align*}

To guarantee the descent of $\mathbb{V}^k$, the following constraints need to be satisfied
\begin{subequations}\label{eq.step-cond}
\begin{align}
    &\alpha_k\leq \frac{1}{2L_F+4L_fL_y+\frac{L_fL_{yx}}{L_y\eta}}\\
    &T\rho_g\beta_k\geq 2L_fL_y\alpha_k + \frac{\eta L_{yx}\tilde{C}_f^2}{4}\alpha_k^2\\
    &\beta_k\leq \frac{2}{\mu_g+\ell_{g,1}}.
\end{align}
\end{subequations}

Finally, we define (with $\rho_g:=\frac{2\mu_g\ell_{g,1}}{\mu_g+\ell_{g,1}}$)
\begin{equation}\label{eq.step-cond-2}
\bar\alpha_1=\frac{1}{2L_F+4L_fL_y+\frac{L_fL_{yx}}{L_y\eta}}, ~~~\bar\alpha_2=\frac{8T\rho_g}{(\mu_g+\ell_{g,1})(8L_fL_y + \eta L_{yx}\tilde{C}_f^2\bar\alpha_1)}
\end{equation}
and, to satisfy the condition \eqref{eq.step-cond}, we select the following stepsizes as
\begin{equation}\label{eq.step-cond-3}
    \alpha_k=\min\{\bar\alpha_1, \bar\alpha_2, \frac{\alpha}{\sqrt{K}}\},~~~~~~~\beta_k=\frac{8L_fL_y + \eta L_{yx}\tilde{C}_f^2\bar\alpha_1}{4T\rho_g}\alpha_k.
\end{equation}

With the above choice of stepsizes, \eqref{eq.thm1-pf-2} can be simplified as
\begin{align*}\label{eq.thm1-pf-3}
    \EE[\mathbb{V}^{k+1}] - \EE[\mathbb{V}^k]
    \leq &- \frac{\alpha_k}{2}\EE[\|\nabla F(x^k)\|^2] + c_1\alpha_k^2\sigma_{g,1}^2 + \alpha_kb_k^2 + c_2\alpha_k^2\tilde\sigma_f^2\numberthis
\end{align*}
where the constants $c_1$ and $c_2$ are defined as
\begin{align*}\label{eq.const_c1c2}
    &c_1=\frac{L_f}{L_y}\Big(1 + 2L_fL_y\bar\alpha_1 + \frac{\eta L_{yx}\tilde{C}_f^2}{4}\bar\alpha_1^2 \Big)\Big(\frac{8L_fL_y + \eta L_{yx}\tilde{C}_f^2\bar\alpha_1}{4\rho_g}\Big)^2\frac{1}{T}\\
    &c_2=\Big(\frac{L_F}{2} + L_fL_y +\frac{L_{yx}L_f}{4\eta L_y}\Big).\numberthis
\end{align*}

Then telescoping leads to
\begin{align*}\label{eq.thm1-pf-4}
    \frac{1}{K}\sum_{k=0}^{K-1}\EE[\|\nabla F(x^k)\|^2]
    &\leq\frac{\mathbb{V}^0+\sum_{k=0}^{K-1}\alpha_kb_k^2 + c_1\alpha_k^2\sigma_{g,1}^2 + c_2T\beta_k^2\tilde\sigma_{f}^2}{\frac{1}{2}\sum_{k=0}^{K-1}\alpha_k}\\
    &\leq \frac{2\mathbb{V}^0}{K\min\{\bar\alpha_1, \bar\alpha_2\}} + \frac{2\mathbb{V}^0}{\alpha\sqrt{K}} + 2b_k^2 +  \frac{2c_1\alpha}{\sqrt{K}}\sigma_{g,1}^2 + \frac{2c_2\alpha}{\sqrt{K}}\tilde\sigma_{f}^2.\numberthis
\end{align*}
To obtain the best $\kappa$-dependence, we choose the balancing constant $\eta=\frac{L_f}{L_y}={\cal O}(\kappa)$, and then we can get $\bar\alpha_1={\cal O}(\kappa^{-3})$, $\bar\alpha_2={\cal O}(T\kappa^{-3})$, $c_1={\cal O}(\kappa^9/T)$, $c_2={\cal O}(\kappa^3)$. To obtain $b_k^2=\frac{1}{\sqrt{K}}$, we need $N={\cal O}(\kappa\log K)$. 
Select $\alpha=\Theta(\kappa^{-2.5})$ and $T={\cal O}(\kappa^4)$,
we are able to get
\begin{align*}
    \frac{1}{K}\sum_{k=0}^{K-1}\EE[\|\nabla F(x^k)\|^2] = {\cal O}\left(\frac{\kappa^3}{K} + \frac{\kappa^{2.5}}{\sqrt{K}}\right).
\end{align*}
To achieve $\varepsilon$-optimal solution, we need $K={\cal O}(\kappa^5\varepsilon^{-2})$, and the number of evaluations of $h_{f}^k, h_{g}^k$ are ${\cal O}(\kappa^5\varepsilon^{-2}), {\cal O}(\kappa^9\varepsilon^{-2})$ respectively.

\section{Proof for stochastic min-max problem}
Recall that the lower-level function for the min-max problem is $g(x,y;\phi)=-f(x,y;\xi)$. Then we rewrite the bilevel problem \eqref{opt0} as 
\begin{subequations}\label{opt0-minmax}
	\begin{align}
	&\min_{x\in \mathbb{R}^d}~~~F(x):=\mathbb{E}_{\xi}\left[f\left(x, y^*(x);\xi\right)\right]\\
	&~{\rm s. t.}~~~~~y^*(x)= \argmin_{y\in \mathbb{R}^{d'}}-\mathbb{E}_{\xi}[f(x, y;\xi)].
\end{align} 
\end{subequations}

In this case, the bilevel gradient in \eqref{grad-deter-2} reduces to
\begin{equation}
	\nabla F(x):=\nabla_xf\big(x, y^*(x)\big)+\nabla_xy^*(x)^{\top}\nabla_y f\big(x, y^*(x)\big)=\nabla_xf\big(x, y^*(x)\big) 
\end{equation}
where the second equality follows from the optimality condition of the lower-level problem, i.e., $\nabla_yf(x,y^*(x))=0$. 
We approximate $\nabla F(x)$ on a  vector $y$ in place of $y^*(x)$, denoted as $\overline{\nabla}f(x,y):=\nabla_xf\big(x, y\big)$. 
Therefore, the alternating stochastic gradients for this special case are given by
	\begin{equation}
    h_g^{k,t} =-\nabla_yf(x^k,y^{k,t};\xi_1^k)~~~{\rm and}~~~    h_f^k = \nabla_xf(x^k,y^{k+1};\xi_2^k).
\end{equation}

\subsection{Verifying lemmas}
 We make the following assumptions that are counterparts of Assumptions 1--3, most of which are common in the min-max optimization literature \citep{rafique2018non,thekumparampil2019efficient,nouiehed2019solving,lin2020gradient}. 

\vspace{0.1cm}
\begin{assumption}[Lipschitz continuity]
Assume that $f, \nabla f, \nabla^2f$ are  $\ell_{f,0}$, $\ell_{f,1}, \ell_{f,2}$-Lipschitz continuous; that is, for $z_1:=[x_1;y_1]$, $z_2:=[x_2;y_2]$, we have $\|f(x_1,y_1)-f(x_2,y_2)\|\leq \ell_{f,0}\|z_1-z_2\|, \|\nabla f(x_1,y_1)-\nabla f(x_2,y_2)\|\leq \ell_{f,1}\|z_1-z_2\|, \|\nabla^2 f(x_1,y_1)-\nabla^2 f(x_2,y_2)\|\leq \ell_{f,2}\|z_1-z_2\|$.
\end{assumption}

\begin{assumption}[Strong convexity of $f$ in $y$]
For any fixed $x$, $f(x,y)$ is $\mu_f$-strongly convex in $y$.
\end{assumption}

Assumptions 1 and 2 together ensure that the first- and second-order derivations of $f(x,y)$ as well as the solution mapping $y^*(x)$ are well-behaved. Define the condition number $\kappa:={\ell_{f,1}}/{\mu_f}$.

\begin{assumption}[Stochastic derivatives]
The stochastic gradient $\nabla f(x,y;\xi)$ is an unbiased estimators of $\nabla f(x,y)$; and its variances is bounded by $\sigma_f^2$.
\end{assumption}


Next we re-derive Lemmas \ref{lemmasm}, \ref{lemma_lip} and \ref{lemm_bias_var} for this special case.
\begin{lemma}[Counterparts of Lemmas \ref{lemmasm}, \ref{lemma_lip} and \ref{lemm_bias_var}]\label{lemma_lip_min-max}
Under Assumptions 1--3, we have
\begin{align*}
\text{(Lemma \ref{lemmasm})}~~~    &\|\nabla y^*(x_1)-\nabla y^*(x_2)\|\leq L_{yx}\|x_1-x_2\|,~~~\EE[\|h_f^k\|^2|{\cal F}_k']\leq \tilde{C}_f^2\\
\text{(Lemma \ref{lemma_lip})}~~~     &\|\overline{\nabla}f(x,y^*(x))-\overline{\nabla}f(x,y)\|\leq L_f\|y^*(x)-y\|\\
& \|\nabla F(x_1)-\nabla F(x_2)\|\leq L_F\|x_1-x_2\|, ~~~\|y^*(x_1)-y^*(x_2)\|\leq L_y\|x_1-x_2\|\\
\text{(Lemma \ref{lemm_bias_var})}~~~    & \bar{h}_f^k=\overline{\nabla}f(x^k,y^{k+1}),~~~\EE[\|h_f^k-\bar{h}_f^k\|^2|{\cal F}_k']\leq \tilde\sigma_f^2
\end{align*}
where the constants are defined as
\begin{align*}
    &L_{yx}=\frac{\ell_{f,2}+\ell_{f,2}L_y}{\mu_f} + \frac{\ell_{f,1}(\ell_{f,2}+\ell_{f,2}L_y)}{\mu_f^2}={\cal O}(\kappa^3), ~~~ \tilde{C}_f^2=\ell_{l,0}^2 + \sigma_f^2\\
    &L_f=\ell_{f,1}={\cal O}(1), ~~~ L_F=(\ell_{f,1}+\frac{\ell_{f,1}^2}{\mu_f})={\cal O}(\kappa), ~~~ L_y=\frac{\ell_{f,1}}{\mu_f}={\cal O}(\kappa),~~~\tilde\sigma_f^2 = \sigma_f^2.
\end{align*}
\end{lemma}
\begin{proof}
We first calculate $L_f$ by
\begin{align*}
    \|\overline{\nabla}f(x,y^*(x))-\overline{\nabla}f(x,y)\|&=\|\nabla_xf(x,y^*(x))-\nabla_xf(x,y)\|\\
  &  \leq \ell_{f,1}\|y^*(x)-y\|:=L_f\|y^*(x)-y\|. \numberthis
\end{align*}
We then calculate $L_F$ by
\begin{align*}
   \|\nabla F(x_1)-\nabla F(x_2)\| &=\|\nabla_xf(x_1,y^*(x_1))-\nabla_xf(x_2,y^*(x_2))\|\\
    &\leq\|\nabla_xf(x_1,y^*(x_1))-\nabla_xf(x_2,y^*(x_1))\| + \|\nabla_xf(x_2,y^*(x_1))-\nabla_xf(x_2,y^*(x_2))\|\\
    &\leq\ell_{f,1}\|x_1-x_2\| + \ell_{f,1}\|y^*(x_1)-y^*(x_2)\|\\
    &\leq \left(\ell_{f,1}+\frac{\ell_{f,1}^2}{\mu_f}\right)\|x_1-x_2\|:=L_F\|x_1-x_2\|. \numberthis
\end{align*}
The calculation of $L_y, L_{yx}$ follows the proof of Lemma \ref{lemmasm} and Lemma  \ref{lemma_lip}, and $\tilde\sigma_f^2, \tilde{C}_f^2, \sigma_g^2$ follows from the fact $h_f^k=\nabla_xf(x^k,y^{k+1};\xi_2^k)$, $h_g^{k,t} = -\nabla_yf(x^k, y^{k,t};\xi_2^{k,t})$. 
\end{proof}


\subsection{Reduction from Theorem \ref{theorem1} to Proposition \ref{prop.min-max}}
In the min-max case, we apply Theorem \ref{theorem1} with $\eta=1$. We define 
\begin{align*}
\bar\alpha_1=\frac{1}{2L_F+4L_fL_y+\frac{L_fL_{yx}}{L_y}}, ~~~\bar\alpha_2=\frac{8T\rho_g}{(\mu_g+\ell_{g,1})(8L_fL_y + L_{yx}\tilde{C}_f^2\bar\alpha_1)}
\end{align*}
and, to satisfy the condition \eqref{eq.step-cond}, we select
\begin{align*}
    &\alpha_k=\min\{\bar\alpha_1, \bar\alpha_2, \frac{\alpha}{\sqrt{K}}\}~~~{\rm and}~~~\beta_k=\frac{8L_fL_y + L_{yx}\tilde{C}_f^2\bar\alpha_1}{4T\rho_g}\alpha_k.
\end{align*}

With the above choice of stepsizes, \eqref{eq.thm1-pf-4} can be simplified as
\begin{align*}
    \frac{1}{K}\sum_{k=0}^{K-1}\EE[\|\nabla F(x^k)\|^2]
    &\leq \frac{2\mathbb{V}^0}{K\min\{\bar\alpha_1, \bar\alpha_2\}} + \frac{2\mathbb{V}^0}{\alpha\sqrt{K}} + \frac{2c_1\alpha}{\sqrt{K}}\sigma_f^2 + \frac{2c_2\alpha}{\sqrt{K}}\sigma_f^2,\numberthis
\end{align*}
where the constants can be defined as
\begin{align*}
    &c_1=\frac{L_f}{L_y}\Big(1 + 2L_fL_y\alpha_k + \frac{L_{yx}\tilde{C}_f^2}{4}\alpha_k^2 \Big)\Big(\frac{8L_fL_y + \eta L_{yx}\tilde{C}_f^2\bar\alpha_1}{4\rho_g}\Big)^2\frac{1}{T}={\cal O}(\frac{\kappa^3}{T})\\
    &c_2=\Big(\frac{L_F}{2} + L_fL_y +\frac{L_{yx}L_f}{4L_y}\Big)={\cal O}(\kappa^2).
\end{align*}

Note that $\bar\alpha_1={\cal O}(\kappa^{-2})$, $\bar\alpha_2={\cal O}(T\kappa^{-2})$. Select $\alpha=\Theta(\kappa^{-1})$, $T=\Theta(\kappa)$, then
\begin{align}
    \frac{1}{K}\sum_{k=0}^{K-1}\EE[\|\nabla F(x^k)\|^2]={\cal O}\left(\frac{\kappa^2}{K} + \frac{\kappa}{\sqrt{K}}\right).
\end{align}
To achieve $\varepsilon$-accuracy, we need $K={\cal O}(\kappa^2\varepsilon^{-2})$. And the number of gradient evaluations for $h_f^k, h_g^k$ are ${\cal O}(\kappa^2\varepsilon^{-2})$, ${\cal O}(\kappa^3\varepsilon^{-2})$ respectively.



\section{Proof for stochastic compositional problem}

Recall that in the stochastic compositional problem, the upper-level function is defined as $f(x,y;\xi):=f(y;\xi)$, and the lower-level function is defined as $g(x,y;\phi):=\frac{1}{2}\|y-h(x;\phi)\|^2$. Then we rewrite the bilevel problem \eqref{opt0} as 
\begin{subequations}\label{opt0-comp}
	\begin{align}
	&\min_{x\in \mathbb{R}^d}~~~F(x):=\mathbb{E}_{\xi}\left[f\left(y^*(x);\xi\right)\right]\\
	&~{\rm s. t.}~~~~~y^*(x)= \argmin_{y\in \mathbb{R}^{d'}} \frac{1}{2}\mathbb{E}_{\phi}[\|y-h(x;\phi)\|^2].
\end{align} 
\end{subequations}

In this case, the bilevel gradient in \eqref{grad-deter-2} reduces to
\begin{align}
	\nabla F(x):&=\nabla_xf\big( y^*(x)\big)-\nabla_{xy}^2g(x, y^*(x))\!\left[\nabla_{yy}^2 g(x, y^*(x))\right]^{-1}\nabla_y f(x, y^*(x))\nonumber\\
	&=\nabla h(x;\phi)^{\top}\nabla_y f(y^*(x))
\end{align}
where we use the fact that $\nabla_{yy}g(x,y;\phi)=\mathbf{I}_{d'\times d'}$ and $\nabla_{xy}g(x,y;\phi)=-\nabla h(x;\phi)^{\top}$. 

Similar to Section \ref{sec.scg}, we again evaluate $\nabla F(x)$ on a certain vector $y$ in place of $y^*(x)$, which is denoted as $\overline{\nabla}f(x,y)=\nabla h(x)\nabla f(y)$. 
Therefore, the alternating stochastic gradients $h_f^k, h_g^{k,t}$ for this special case are much simpler, given by
\begin{align}
    h_g^{k,t}=y^{k,t}-h(x^k;\phi^{k,t})~~~{\rm and}~~~ 
    h_f^k=\nabla h(x^k;\phi^k)\nabla f(y^{k+1};\xi^k).
\end{align}
It can be observed that $h_f^k$ is an unbiased estimate of $\overline{\nabla}f(x^k,y^{k+1})$, that is, $\bar{h}_f^k=\overline{\nabla}f(x,y), b_k=0$.

\subsection{Verifying lemmas}


We make the following assumptions that are counterparts of Assumptions 1--3, all of which are common in compositional optimization literature \citep{wang2017mp,wang2017jmlr,ghadimi2020jopt,zhang2019nips,chen2020scsc}. 

\begin{assumption}[Lipschitz continuity]
Assume that $f, \nabla f, h, \nabla h$ are respectively $\ell_{f,0}$, $\ell_{f,1}, \ell_{h,0}, \ell_{h,1}$-Lipschitz continuous; that is, for $z_1:=[x_1;y_1]$, $z_2:=[x_2;y_2]$, we have $\|f(x_1,y_1)-f(x_2,y_2)\|\leq \ell_{f,0}\|z_1-z_2\|, \|\nabla f(x_1,y_1)-\nabla f(x_2,y_2)\|\leq \ell_{f,1}\|z_1-z_2\|$, $\|h(x_1)-h(x_2)\|\leq \ell_{h,0}\|x_1-x_2\|, \|\nabla h(x_1) - \nabla h(x_2)\|\leq \ell_{h,1}\|x_1-x_2\|$.
\end{assumption}

Note that the Lipschitz continuity of $\nabla g, \nabla^2g$ in Assumption 1 can be implied by the Lipschitz continuity of $h, \nabla h$ in the above assumption. 
Assumption 2 is automatically satisfied for stochastic compositional problems since $\nabla_{yy}g(x,y;\phi)=\mathbf{I}_{d'\times d'}$ and the condition number $\kappa:=1$.

\begin{assumption}[Stochastic derivatives]
The stochastic quantities $\nabla f(x,y;\xi)$, $h(x;\phi)$, $\nabla h(x;\phi)$ are unbiased estimators of $\nabla f(x,y)$, $h(x)$, $\nabla h(x)$, respectively; and their variances are bounded by $\sigma_f^2, \sigma_{h,0}^2$, $\sigma_{h,1}^2$, respectively.
\end{assumption}

The unbiasedness and bounded variance of $\nabla g(x,y;\phi)$, $\nabla^2g(x, y, \phi)$ in Assumption 3 can be implied by the unbiasedness and bounded variance of $h(x;\phi)$, $\nabla h(x;\phi)$. 

Next we re-derive Lemmas \ref{lemmasm}, \ref{lemma_lip} and \ref{lemm_bias_var} for this special case.

\begin{lemma}[Counterparts of Lemmas \ref{lemmasm}, \ref{lemma_lip} and \ref{lemm_bias_var}]\label{lemma_lip_comp}
Under Assumptions 1--3, we have
\begin{align*}
\text{(Lemma \ref{lemmasm})}~~~       &\|\nabla y^*(x_1)-\nabla y^*(x_2)\|\leq L_{yx}\|x_1-x_2\|,~~~\EE[\|h_f^k\|^2|{\cal F}_k']\leq \tilde{C}_f^2\\
\text{(Lemma \ref{lemma_lip})}~~~     &\|\overline{\nabla}f(x,y^*(x))-\overline{\nabla}f(x,y)\|\leq L_f\|y^*(x)-y\|\\
&\|\nabla F(x_1)-\nabla F(x_2)\|\leq L_F\|x_1-x_2\|,~~~\|y^*(x_1)-y^*(x_2)\|\leq L_y\|x_1-x_2\|\\
\text{(Lemma \ref{lemm_bias_var})}~~~     &\EE[\|h_f^k-\bar{h}_f^k\|^2|{\cal F}_k']\leq \tilde\sigma_f^2, ~~~\bar{h}_f^k=\overline{\nabla}f(x^k,y^{k+1})
\end{align*}
where the constants are defined as
\begin{align*}
&L_{f} = \ell_{h,0}\ell_{f,1}, ~~~ L_y =\ell_{h,0}, ~~~ L_{F} =\ell_{h,0}^2\ell_{f,1} + \ell_{f,0}\ell_{h,1}, ~~L_{yx} =\ell_{h,1}\\
&\tilde\sigma_f^2=\ell_{h,0}^2\sigma_f^2+(\ell_{f,0}^2 + \sigma_f^2)\sigma_{h,1}^2, ~~~\tilde{C}_f^2=(\ell_{f,0}^2+\sigma_f^2)(\ell_{h,0}^2+\sigma_{h,1}^2). \numberthis
\end{align*}
\end{lemma}
\begin{proof}
We first calculate $L_f$ by
\begin{align*}
    \|\overline{\nabla}f(x,y^*(x))\!-\!\overline{\nabla}f(x,y)\|&=\|\nabla h(x)\nabla f(y^*(x))-\nabla h(x)\nabla f(y)\|\\
    &\leq\|\nabla h(x)\|\|\nabla f(y^*(x))-\nabla f(y)\|\\
    &\leq \ell_{h,0}\ell_{f,1}\|y^*(x)-y\|:=L_f\|y^*(x)-y\|.
\end{align*}

We then calculate $L_F$ by
\begin{align*}
    \|\nabla F(x_1)-\nabla F(x_2)\|&=\|\nabla h(x_1)\nabla f(h(x_1))-\nabla h(x_2)\nabla f(h(x_2))\|\\
    &\leq\|\nabla h(x_1)\|\|\nabla f(h(x_1))\!-\!\nabla f(h(x_2))\|\!+\!\|\nabla f(h(x_2))\|\|\nabla h(x_1)\!-\!\nabla h(x_2)\|\\
    &\leq\ell_{h,0}^2\ell_{f,1}\|x_1-x_2\| + \ell_{f,0}\ell_{h,1}\|x_1-x_2\|\\
    &:=L_F\|x_1-x_2\|. \numberthis
 \end{align*} 
 
 We then calculate $L_y$ and $L_{yx}$ by
 \begin{align*}   
    \|y^*(x_1)-y^*(x_2)\|&=\|h(x_1)-h(x_2)\|
    \leq \ell_{h,0}\|x_1-x_2\| := L_y\|x_1-x_2\|\\
        \|\nabla y^*(x_1)-\nabla y^*(x_2)|\|&=\|\nabla h(x_1)-\nabla h(x_2)\|\leq \ell_{h,1}\|x_1-x_2\|:=L_{yx}\|x_1-x_2\|.
\end{align*}  
    
We then calculate $\tilde\sigma_f^2$ by
\begin{align*}    
    \EE[\|h_f^k-\bar{h}_f^k\|^2|{\cal F}_k']&\leq \EE[\|\nabla h(x^k;\phi_2^k)\nabla f(y^{k+1};\xi^k)-\nabla h(x^k)\nabla f(y^{k+1})\|^2|{\cal F}_k']\\
    &\leq \EE[\|\nabla f(y^{k+1};\xi^k)\|^2\|\nabla h(x^k;\phi_2^k)\!-\!\nabla h(x^k)\|^2|{\cal F}_k']\\
    &\quad + \EE[\|\nabla h(x^k)\|^2\|\nabla f(y^{k+1};\xi^k)\!-\!\nabla f(y^{k+1})\|^2|{\cal F}_k']\\
    &\leq (\ell_{f,0}^2+\sigma_f^2)\sigma_{h,1}^2 + \ell_{h,0}^2\sigma_f^2:=\tilde\sigma_f^2. \numberthis
\end{align*}

We then calculate $\tilde{C}_f^2$ by
\begin{align*}  
    \EE[\|h_f^k\|^2|{\cal F}_k']&=\EE[\|\nabla h(x^k;\phi_2^k)\nabla f(y^{k+1};\xi^k)\|^2|{\cal F}_k']\\
    &\leq\EE[\|\nabla f(y^{k+1};\xi^k)\|^2|{\cal F}_k']\EE[\|\nabla h(x^k;\phi_2^k)\|^2|{\cal F}_k']\\
    &\leq(\ell_{f,0}^2 + \sigma_f^2)(\ell_{h,0}^2 + \sigma_{h,1}^2):=\tilde{C}_f^2. \numberthis
\end{align*}
\end{proof}

\subsection{Reduction from Theorem \ref{theorem1} to Proposition \ref{prop.comp}}
In the compositional case, we apply Theorem \ref{theorem1} by setting $T=1,\alpha=1, \eta=\frac{1}{\ell_{h,1}}$.  We define 
\begin{align*}
\bar\alpha_1=\frac{1}{6\ell_{h,0}^2\ell_{f,1} + 2\ell_{f,0}\ell_{h,1}+\ell_{f,1}\ell_{h,1}^2}, ~~~\bar\alpha_2=\frac{8}{(\mu_g+\ell_{g,1})(8\ell_{f,1}\ell_{h,0}^2 + \tilde{C}_f^2\bar\alpha_1)}
\end{align*}
and, to satisfy the condition \eqref{eq.step-cond}, we select
\begin{align*}
\alpha_k=\min\left\{\bar\alpha_1, \bar\alpha_2, \frac{\alpha}{\sqrt{K}}\right\}~~~{\rm and}~~~
\beta_k=\frac{8\ell_{f,1}\ell_{h,0}^2 + \tilde{C}_f^2\bar\alpha_1}{4}\alpha_k. \numberthis
\end{align*}
And the constants $c_1, c_2$ in \eqref{eq.const_c1c2} reduce to
\begin{align*}
    &c_1=\ell_{f,1}\Big(1 + 2\ell_{f,1}\ell_{h,0}^2\bar\alpha_1+ \frac{\tilde{C}_f^2}{4}\bar\alpha_1^2 \Big)\Big(\frac{8\ell_{h,0}^2\ell_{f,1} + \tilde{C}_f^2\bar\alpha_1}{4}\Big)^2\\
    &c_2=\Big(\frac{\ell_{h,0}^2\ell_{f,1} + \ell_{f,0}\ell_{h,1}}{2} + \ell_{h,0}^2\ell_{f,1} +\frac{\ell_{f,1}\ell_{h,1}^2}{4}\Big).\numberthis
\end{align*}

We apply \eqref{eq.thm1-pf-4} and get
\begin{align*}\label{eq.thm1-pf-5}
    \frac{1}{K}\sum_{k=0}^{K-1}\EE[\|\nabla F(x^k)\|^2]
    \leq \frac{2\mathbb{V}^0}{K\min\{\bar\alpha_1, \bar\alpha_2\}} + \frac{2\mathbb{V}^0}{\alpha\sqrt{K}} + \frac{2c_1}{\sqrt{K}}\sigma_{h,1}^2 + \frac{2c_2}{\sqrt{K}}\tilde\sigma_f^2 = {\cal O}\left(\frac{1}{\sqrt{K}}\right)\numberthis
\end{align*}
from which the proof is complete. 

\section{Proof for actor-critic method}

Recall the state feature mapping $\phi(\cdot):\mathcal{S}\xrightarrow[]{}\mathbb{R}^{d'}$. 
Define 
\begin{subequations}\label{def:A-b}
\begin{align}
A_{\theta, \phi} &\coloneqq \E_{ s\sim \mu_{\theta}, s'\sim \mathcal{P}_{\pi_\theta}}[\phi(s)(\gamma\phi(s')-\phi(s))^\top],\\
b_{\theta, \phi} &\coloneqq \E_{ s\sim \mu_{\theta}, a \sim \pi_\theta, s'\sim \mathcal{P}}[r(s,a,s')\phi(s)].
\end{align}
\end{subequations}
It is known that for a given $\theta$, the stationary point $y^*(\theta)$ of the TD update in \eqref{eq:td-update} satisfies
\begin{equation}
      A_{\theta, \phi} y^*(\theta) +b_{\theta, \phi} = 0.
\end{equation}
 
Due to the special nature of the policy gradient, we make the following assumptions that will lead to the counterparts of Lemmas \ref{lemmasm}, \ref{lemma_lip} and \ref{lemm_bias_var} in reinforcement learning. 
These assumptions are mostly common in analyzing actor-critic method with linear value function approximation \citep{wu2020finite,xu2020improving,fu2020single}. 
 
\begin{assumption}\label{assumption:A}
For all $s \in \mathcal{S}$, the feature vector $\phi(s)$ is normalized so that $\|\phi(s)\|_2 \leq 1$. 
 For all eligible $\theta$, $A_{\theta, \phi}$ is negative definite and its maximum eigenvalue is upper bounded by constant $-\lambda$.
\end{assumption}

Assumption \ref{assumption:A} is common in analyzing TD with linear function approximation; see e.g.,  \citep{JBTD,xu2020reanalysis,wu2020finite}. With this assumption, $A_{\theta, \phi}$ is invertible, so we have $y^*(\theta)=-A_{\theta, \phi}^{-1}b_{\theta, \phi}$. 
Defining $R_y \coloneqq r_{\max}/\lambda$, we have $\|y^*(\theta)\|_2   \leq R_y$. 
It justifies the projection introduced in the critic update \eqref{eq:td-update}. 

\begin{assumption}\label{assumption:omega}
For any {\small$\theta, \theta' \in \mathbb{R}^{d}$, $s \in \mathcal{S}$} and {\small$a \in \mathcal{A}$}, there exist  constants $C_\psi, L_\psi, L_\pi$ such that: i) $\|\psi_\theta(s,a)\|_2 \leq C_\psi$; ii) $\|\psi_\theta(s,a) - \psi_{\theta'}(s,a)\|_2 \leq L_\psi\|\theta-\theta'\|_2$; iii) $\left|\pi_\theta(a|s)-\pi_{\theta'}(a|s)\right| \leq L_\pi\|\theta-\theta'\|_2$.  
\end{assumption}


Assumption \ref{assumption:omega} is common in analyzing policy gradient-type algorithms which has also been made by e.g., \citep{zhang2019global,agarwal2019optimality}. 
This assumption holds for many policy parameterization methods such as tabular softmax policy \citep{agarwal2019optimality}, Gaussian policy \citep{doya2000reinforcement} and Boltzmann policy \citep{konda1999actor}.

\begin{assumption}\label{assumption:mu}
For any {\small$\theta, \theta' \in \mathbb{R}^{d}$}, there exist  constants such that: i) $\|\nabla \mu_\theta(s)\|_2 \leq C_\mu$; ii) $\|\nabla \mu_\theta(s) -\nabla \mu_{\theta'}(s)\|_2 \leq L_{\mu,1}\|\theta-\theta'\|_2$; iii) $\left|\mu_\theta(s)-\mu_{\theta'}(s)\right| \leq L_{\mu,0}\|\theta-\theta'\|_2$.  
\end{assumption}

Assumption \ref{assumption:mu} is the counterpart of Assumption \ref{assumption:omega} that is made for the stationary distribution $\mu_\theta(a|s)$. 
Note that the existence of $\nabla \mu_\theta(s)$ has been shown in \citep{baxter2001jair}.
In this case, under Assumption \ref{assumption:omega}, i) and iii) of Assumption \ref{assumption:mu} can be obtained from the sensitivity analysis of Markov chain; see e.g.,  \citep[Theorem 3.1]{mitrophanov2005sensitivity}. 
While we cannot provide a justification of (ii), we found it necessary to ensure the smoothness of the lower-level critic solution $y^*(\theta)$.

\begin{assumption}\label{assumption:MDP}
For any $\theta$, the Markov chain under $\pi_\theta$ and transition kernel $\mathcal{P}(\cdot|s,a)$ is irreducible and aperiodic. Then there exist constants $\kappa > 0$ and $\rho \in (0,1)$ such that
\begin{equation}\label{eq:uniform-geom}
\sup_{s \in \mathcal{S}}~~~d_{TV}\left(\mathbb{P}(s_t \in \cdot|s_0=s, \pi_\theta),\mu_{\theta}\right)
    \leq \kappa \rho^t,~~~ \forall t 
\end{equation}
where $\mu_{\theta}$ is the stationary state distribution under $\pi_\theta$, and $s_t$ is the state of Markov chain at time $t$. 
\end{assumption}

Assumption \ref{assumption:MDP} assumes the Markov chain mixes at a geometric rate; see also \citep{JBTD,xu2020reanalysis}.

We define the critic approximation error as 
\begin{equation}\label{eq:epsilon_app}
    \epsilon_{app} \coloneqq \max_{\theta \in \mathbb{R}^{d}} \sqrt{\E_{s \sim \mu_{\theta}}|V_{\pi_\theta}(s)-\hat{V}_{y^*_{\theta}}(s)|^2}.
\end{equation}
This error captures the quality of the critic function approximation; see also \citep{qiu2019finite,wu2020finite,xu2020improving}. It becomes zero when the value function $V_{\pi_\theta}$ belongs to the linear function space for any $\theta$.

\subsection{Auxiliary lemmas}
We give a proposition regarding the $L_F$-Lipschitz of the policy gradient under proper assumptions.
\begin{proposition}[Smoothness of policy gradiemt \citep{zhang2019global}]\label{prop:Lj-lip}
Suppose Assumption \ref{assumption:omega} holds. For any $\theta, \theta' \in \mathbb{R}^{d}$, we have $\|\nabla F(\theta) - \nabla F(\theta')\|_2 \leq L_F \|\theta - \theta'\|_2$, where $L_F$ is a positive constant.
\end{proposition}

We provide a justification for Lipschitz continuity of $y^*(\theta)$ in the next proposition.
\begin{proposition}[Lipschitz continuity of $y^*(\theta)$]\label{proposition:omega-lipschitz}
Suppose Assumption \ref{assumption:omega} and \ref{assumption:MDP} hold. For any $\theta_1, \theta_2 \in \mathbb{R}^{d}$, we have $\|y^*(\theta_1) - y^*(\theta_2)\|_2 \leq L_{y} \|\theta_1-\theta_2\|_2$, where $L_{y}$ is a positive constant.
\end{proposition}
\textit{Proof}. We use $y^*_1, y^*_2, A_1$, $A_2$, $b_1$ and $b_2$ as shorthand notations of $y^*(\theta_1)$, $y^*(\theta_2)$, $A_{\pi_{\theta_1}}$, $A_{\pi_{\theta_2}}$, $b_{\pi_{\theta_1}}$ and $b_{\pi_{\theta_2}}$ respectively. By Assumption \ref{assumption:A}, $A_{\theta, \phi}$ is invertible for any $\theta \in \mathbb{R}^d$, so we can write $y^*(\theta) = -A_{\theta, \phi}^{-1} b_{\theta, \phi}$. Then we have
\begin{align*}\label{eq:w1*-w2*}
    \|y^*_1 - y^*_2 \|_2
    &=\| -A_1^{-1}b_1 +  A_2^{-1}b_2\|_2 \\
    &=\| -A_1^{-1}b_1 - A_1^{-1}b_2 + A_1^{-1}b_2 + A_2^{-1}b_2 \|_2 \\
    &=\| -A_1^{-1}(b_1 - b_2) - (A_1^{-1} - A_2^{-1})b_2 \|_2 \\
    &\leq \|A_1^{-1}(b_1 - b_2)\|_2 + \| (A_1^{-1} -  A_2^{-1})b_2\|_2 \\
    &\leq \|A_1^{-1}\|_2\|b_1 - b_2 \|_2 + \| A_1^{-1} -  A_2^{-1} \|_2\|b_2\|_2 \\
    &= \|A_1^{-1}\|_2\|b_1 - b_2 \|_2 + \| A_1^{-1}(A_2-A_1)A_2^{-1} \|_2\|b_2\|_2 \\
    &\leq \|A_1^{-1}\|_2\|b_1 - b_2 \|_2 + \| A_1^{-1} \|_2 \|A_2^{-1} \|_2 \|b_2\|_2 \|(A_2-A_1)\|_2  \\
    &\leq \lambda^{-1} \left\|b_1 - b_2 \right\|_2 + \lambda^{-2}r_{\max} \left\|A_1 - A_2 \right\|_2, \numberthis
\end{align*}
where the last inequality follows Assumption \ref{assumption:A}, and the fact that
\begin{align}\label{eq:b2}
    \|b_2\|_2 = \left\|\E[r(s,a,s')\phi(s)]\right\|_2 \leq \E\left\|r(s,a,s')\phi(s)\right\|_2 \leq \E\left[|r(s,a,s')|\|\phi(s)\|_2\right] \leq r_{\max}.
\end{align}
Denote $(s^1,a^1,s'^1)$ and $(s^2,a^2,s'^2)$ as samples drawn with $\theta_1$ and $\theta_2$ respectively, i.e. $s^1 \sim \mu_{\theta_1}$, $a^1 \sim \pi_{\theta_1}$, $s'^1 \sim \mathcal{P}$ and  $s^2 \sim \mu_{\theta_2}$, $a^2 \sim \pi_{\theta_2}$, $s'^2 \sim \mathcal{P}$. Then we have
\begin{align*}\label{eq:b1-b2}
    \left\|b_1 - b_2 \right\|_2
    &= \left\| \E\left[r(s^1,a^1,s'^1)\phi(s^1) \right] -  \E\left[r(s^2,a^2,s'^2)\phi(s^2) \right]\right\|_2 \\
    &\leq \sup_{s,a,s'}\|r(s,a,s')\phi(s)\|_2 \|\mathbb{P}((s^1,a^1,s'^1)\in\cdot)-\mathbb{P}((s^2,a^2,s'^2)\in\cdot)\|_{TV}\\
    &\leq r_{\max} \|\mathbb{P}((s^1,a^1,s'^1)\in\cdot)-\mathbb{P}((s^2,a^2,s'^2)\in\cdot)\|_{TV}\\
    &= 2r_{\max} d_{TV} \left( \mu_{\theta_1}\otimes\pi_{\theta_1}\otimes\mathcal{P}, \mu_{\theta_2}\otimes\pi_{\theta_2}\otimes\mathcal{P} \right) \\
    &\leq 2r_{\max} |\mathcal{A}| L_{\pi} (1+\log_\rho \kappa^{-1} + (1-\rho)^{-1})\|\theta_1-\theta_2\|_2, \numberthis
\end{align*}
where the first inequality follows the definition of total variation (TV) norm, and the last inequality follows in \citep[Lemma A.1]{wu2020finite}. Similarly we have:
\begin{align*}\label{eq:a1-a2}
    \left\|A_1 - A_2 \right\|_2
    &\leq 2(1+\gamma) d_{TV} \left( \mu_{\theta_1}\otimes\pi_{\theta_1}, \mu_{\theta_2}\otimes\pi_{\theta_2} \right) \\
    &= (1+\gamma) |\mathcal{A}| L_{\pi} (1+\log_\rho \kappa^{-1} + (1-\rho)^{-1})\|\theta_1-\theta_2\|_2\\
    &:=L_{A,0} \|\theta_1-\theta_2\|_2. \numberthis
\end{align*}
Substituting (\ref{eq:b1-b2}) and (\ref{eq:a1-a2}) into (\ref{eq:w1*-w2*}) completes the proof. \hfill \myQED

We prove the Lipschitz continuity of $\nabla_{\theta}y^*(\theta)$ next, for which we will use the following fact.

\textbf{Fact.} If the functions $f(\theta),g(\theta)$ are bounded by $C_f$ and $C_g$; and are $L_f$- and $L_g$-Lipschitz continuous, then $f(\theta)g(\theta)$ is also bounded by $C_fC_g$ and is $(C_fL_g+C_gL_f)$-Lipschitz continuous. 

\textit{Proof}. Using the Cauchy-Schwartz inequality, it is easy to see that $f(\theta),g(\theta)$ are bounded by $C_fC_g$. In addition, we have that
\begin{align*}
\|f(\theta_1)g(\theta_1)-f(\theta_2)g(\theta_2)\|&= \|f(\theta_1)g(\theta_1)-f(\theta_1)g(\theta_2)+f(\theta_1)g(\theta_2)-f(\theta_2)g(\theta_2)\|\\
&\leq \|f(\theta_1)\| \|g(\theta_1)-g(\theta_2)\|+\|f(\theta_1)-f(\theta_2)\|\|g(\theta_2)\|\\
&\leq (C_fL_g+C_gL_f) \|\theta_1-\theta_2\|_2 
\end{align*}
which implies that $f(\theta),g(\theta)$ is $(C_fL_g+C_gL_f)$-Lipschitz continuous.

\begin{proposition}[Lipschitz continuity of $\nabla_{\theta}y^*(\theta)$]\label{proposition:omega-smooth}
Suppose Assumption \ref{assumption:omega}-\ref{assumption:MDP} hold. For any $\theta_1, \theta_2 \in \mathbb{R}^{d}$, we have $\|\nabla_{\theta}y^*(\theta_1) - \nabla_{\theta}y^*(\theta_2)\|_2 \leq L_{yx} \|\theta_1-\theta_2\|_2$, where $L_{yx}$ is a positive constant.
\end{proposition}

\textit{Proof}.
With $y^*(\theta)=-A_{\theta,\phi}^{-1}b_{\theta,\phi}$, we have
\begin{align}\label{eq.nabla_y*}
    \nabla_{\theta}y^*(\theta)=-\nabla_{\theta}(A_{\theta,\phi}^{-1}b_{\theta,\phi})
    =-A_{\theta,\phi}^{-1}(\nabla_{\theta}A_{\theta,\phi})A_{\theta,\phi}^{-1}b_{\theta,\phi} - A_{\theta,\phi}(\nabla_{\theta}b_{\theta,\phi}).
\end{align}

To validate the Lipschitz continuity of $\nabla_{\theta}y^*(\theta)$, we need to show the boundedness and Lipschitz continuity of $A_{\theta,\phi}^{-1}$, $ b_{\theta,\phi}$, $\nabla_{\theta}A_{\theta,\phi}$ and $\nabla_{\theta}b_{\theta,\phi}$. 

From \eqref{eq:b1-b2} and \eqref{eq:a1-a2}, we have that there exist constants $L_{A,0}$ and $L_{b,0}$ such that $A_{\theta,\phi}$ is $L_{A,0}$-Lipschitz continuous, and $b_{\theta,\phi}$ is $L_{b,0}$-Lipschitz continuous. From Assumption \ref{assumption:A} and \eqref{eq:b2}, we have that there exist constants $C_{A,0}$ and $C_{b,0}$ such that $\|A_{\theta,\phi}\|_2\leq C_{A,0}$, and $\|b_{\theta,\phi}\|_2\leq C_{b,0}$.

In addition, using $A_1$ and $A_2$ as shorthand notations of $A_{\pi_{\theta_1}}$ and $A_{\pi_{\theta_2}}$,  respectively, we have
\begin{align*}\label{eq:A-1}
    \|A_1^{-1} -  A_2^{-1} \|_2
    &= \| A_1^{-1}(A_2-A_1)A_2^{-1} \|_2 \\
    &\leq \| A_1^{-1} \|_2 \|A_2^{-1} \|_2  \|(A_2-A_1)\|_2  \\
    &\leq \lambda^{-2} \left\|A_1 - A_2 \right\|_2\\
    & \stackrel{\eqref{eq:a1-a2}}{\leq}\lambda^{-2} L_{A,0} \|\theta_1-\theta_2\|_2. \numberthis
\end{align*}
Therefore, $A_{\theta,\phi}^{-1}$ is $\lambda^{-2} L_{A,0}$-Lipschitz continuous, and is bounded by $\lambda^{-1}$ due to Assumption \ref{assumption:A}.

For simplicity, denote 
\begin{align}
A(s,s'):=\phi(s)(\gamma\phi(s')-\phi(s))^\top, ~~~~b(s,a,s'):=r(s,a,s')\phi(s)
\end{align}
and then $b_{\theta, \phi} \coloneqq \E_{ s\sim \mu_{\theta}, a \sim \pi_\theta, s'\sim \mathcal{P}}[b(s,a,s')]$ and $A_{\theta, \phi} \coloneqq \E_{s\sim \mu_{\theta}, s'\sim \mathcal{P}_{\pi_\theta}}[A(s,s')]$.
 
Next we analyze $\nabla_{\theta}A_{\theta,\phi}$ and $\nabla_{\theta}b_{\theta,\phi}$, which is given by 
\begin{align*}
    \nabla_{\theta}A_{\theta,\phi} 
    &=\nabla_{\theta}\Big(\sum_{s,a,s'}\mu_{\theta}(s)\pi_{\theta}(a|s)P(s'|s,a)A(s,s')\Big)\\
    &=\sum_{s,a,s'}\left[\nabla_{\theta}\mu_{\theta}(s)\pi_{\theta}(a|s)P(s'|s,a)A(s,s') + \mu_{\theta}(s)\nabla_{\theta}\pi_{\theta}(a|s)P(s'|s,a)A(s,s')\right].\numberthis
\end{align*}

From Assumption \ref{assumption:omega} and \ref{assumption:mu}, $\mu_{\theta}(s), \pi_{\theta}(a|s), \nabla_{\theta}\mu_{\theta}(s), \nabla_{\theta}\pi_{\theta}(a|s)$ are Lipschitz continuous and bounded. Using the \textbf{Fact}, we can show that there exist constants $C_{A,1}$ and $L_{A,1}$ such that $\nabla_{\theta}A_{\theta,\phi}$ is $L_{A,1}$-Lipschitz continuous and bounded by $C_{A,1}$.

Likewise, we have
\begin{align*}
 \nabla_{\theta}b_{\theta,\phi}
    &=\nabla_{\theta}\Big(\sum_{s,a,s'}\mu_{\theta}(s)\pi_{\theta}(a|s)P(s'|s,a)b(s,a,s')\Big)\numberthis\\
    &=\sum\limits_{s,a,s'}\left[\nabla_{\theta}\mu_{\theta}(s)\pi_{\theta}(a|s)P(s'|s,a)b(s,a,s') + 
    \mu_{\theta}(s)\nabla_{\theta}\pi_{\theta}(a|s)P(s'|s,a)b(s,a,s')\right] 
\end{align*}
From Assumption \ref{assumption:omega} and \ref{assumption:mu}, $\mu_{\theta}(s), \pi_{\theta}(a|s), \nabla_{\theta}\mu_{\theta}(s), \nabla_{\theta}\pi_{\theta}(a|s)$ are Lipschitz continuous and bounded. Using the \textbf{Fact}, we are able to show that there exist constants $C_{b,1}$ and $L_{b,1}$ such that $\nabla_{\theta}b_{\theta,\phi}$ is $L_{b,1}$-Lipschitz continuous and bounded by $C_{b,1}$. 

Therefore, since $A_{\theta,\phi}^{-1}$, $ b_{\theta,\phi}$, $\nabla_{\theta}A_{\theta,\phi}$ and $\nabla_{\theta}b_{\theta,\phi}$ are all Lipschitz continuous, using \textbf{Fact}, we can show that $\nabla_{\theta}y^*(\theta)$ in \eqref{eq.nabla_y*} is $L_{yx}$-Lipschitz continuous, where $L_{yx}$ depends on the constants  $C_\mu, C_\psi, L_{\pi}, L_{\mu,0}, L_{\mu,1}, \lambda$ defined in Assumptions \ref{assumption:omega}-\ref{assumption:MDP}. 
\hfill \myQED

\subsection{Convergence of critic variable}\label{section:async-tts-critic-proof}

For brevity, we first define the following notations (cf. $\xi \coloneqq (s,a,s')$):
\begin{align*}
    \hat{\delta}(\xi,y)
    &\coloneqq r(s,a,s') + \gamma \phi(s')^\top y -\phi(s)^\top y, \\
     h_g(\xi,y) 
    &\coloneqq \hat{\delta}(\xi,y) \phi(s), \\
    \overline{h}_g(\theta,y) 
    &\coloneqq \E_{\substack{s \sim \mu_{\theta}, a \sim \pi_\theta, s' \sim \mathcal{P}}} \left[ h_g(\xi,y)\right].
\end{align*}
We also define constant $C_{\delta} \coloneqq r_{\max} + (1+\gamma)\max\{R_{\max}, R_y\}$, and we immediately have
\begin{align}\label{eq:C_delta}
    \| h_g(\xi,y)\|_2\leq |r(\xi)+\gamma \phi(s')^\top y - \phi(s)^\top y|  \leq r_{\max} + (1+\gamma)R_y \leq C_g 
\end{align}
and likewise, we have $\|\overline{h}_g(\xi,y)\|_2  \leq   C_g$. 
 
The critic update can be written compactly as:
\begin{equation}\label{update:async-tts-critic-update}
    y_{k+1} = \Pi_{R_y}\left(y_k + \beta_k g(\xi_k,y_{k})\right),
\end{equation}
where $\xi_k \coloneqq (s_k,a_k,s_k')$ is the sample used to evaluate the stochastic gradient at $k$th update. 

\textit{Proof.} 
Using $y^*(\theta_k)$ as shorthand notation of $y_{\theta_k}^*$, we start with the optimality gap
\begin{align*}\label{eq:tmp1}
    &\|y_{k+1}-y^*(\theta_{k+1})\|_2^2 \\
     &=\|y_{k+1}-y^*(\theta_k)+y^*(\theta_k)-y^*(\theta_{k+1})\|_2^2 \\
    &=\|y_{k+1}-y^*(\theta_k)\|_2^2+\|y^*(\theta_k)-y^*(\theta_{k+1})\|_2^2+2\left\langle y_{k+1}-y^*(\theta_k), y^*(\theta_k)-y^*(\theta_{k+1})\right\rangle. \numberthis
\end{align*}

We first bound 
\begin{align*}\label{eq:tmp1-2}
	\|y_{k+1}-y^*(\theta_k)\|_2^2=& \|\Pi_{R_y}\left(y_k + \beta_k g(\xi_k,y_{k})\right) - y^*(\theta_k)\|_2^2\\
	\leq &\|y_k + \beta_k g(\xi_k,y_{k}) - y^*(\theta_k)\|_2^2\\
	=& \left\|y_k - y^*(\theta_k) \right\|_2^2 + 2 \beta_k \left\langle y_k-y^*(\theta_k), g(\xi_k,y_{k})\right\rangle + \left\| \beta_k g(\xi_k,y_{k}) \right\|_2^2.
	\numberthis
\end{align*}

We first bound $\E[\left\langle y_k-y^*(\theta_k),  g(\theta_k,y_k)\right\rangle |y_k]$ in (\ref{eq:tmp1}) as
\begin{align*}\label{eq:tmp2}
    \E[\left\langle y_k-y^*(\theta_k), g(\theta_k,y_k)\right\rangle | y_k]
    &= \left\langle y_k-y^*(\theta_k), \overline{h}_g(\theta_k,y_k) - \overline{h}_g(\theta_k,y^*(\theta_k))\right\rangle \\
    &= \left\langle y_k-y^*(\theta_k), \E\left[\left(\gamma \phi(s') - \phi(s) \right)^\top (y_k-y^*(\theta_k)) \phi(s)\right] \right\rangle \\
    &= \left\langle y_k-y^*(\theta_k), \E\left[\phi(s)\left(\gamma \phi(s') - \phi(s) \right)^\top\right] (y_k-y^*(\theta_k))  \right\rangle \\
    &= \left\langle y_k-y^*(\theta_k), A_{\pi_{\theta_k}} (y_k-y^*(\theta_k)) \right\rangle \\
    &\leq  -\lambda \|y_k-y^*(\theta_k)\|_2^2, \numberthis
\end{align*}
where the first equality is due to $\overline{h}_g(\theta,y^*_\theta)=A_{\theta, \phi}y^*(\theta) + b=0 $, and the last inequality follows Assumption \ref{assumption:A}. 

Substituting (\ref{eq:tmp2}) into (\ref{eq:tmp1-2}), then taking expectation on both sides of (\ref{eq:tmp1-2}) yields
\begin{align*}\label{eq:tmp2-2}
    \E\|y_{k+1}-y^*(\theta_k)\|_2^2
    &\leq (1-2\lambda \beta_k)\E\left\|y_k - y^*(\theta_k) \right\|_2^2  + C_g^2\beta_k^2 \numberthis
\end{align*}
and plugging into into (\ref{eq:tmp1}) yields
\begin{align*}\label{eq:tmp3}
    \E\|y_{k+1}-y^*(\theta_{k+1})\|_2^2
    &\leq (1-2\lambda \beta_k)\E\left\|y_k - y^*(\theta_k) \right\|_2^2  \\
    &~~~~ + 2 \E\left\langle y_{k+1}-y^*(\theta_k), y^*(\theta_k)-y^*(\theta_{k+1})\right\rangle+  \E\left\| y^*(\theta_k)-y^*(\theta_{k+1})\right\|_2^2 + C_g^2\beta_k^2. \numberthis
\end{align*}


Next we bound the third and fourth terms in (\ref{eq:tmp3}) as
{\small\begin{align*}\label{eq:tmp8}
    &  \E\left\langle y_{k+1}-y^*(\theta_k), y^*(\theta_k)-y^*(\theta_{k+1})\right\rangle \\
    &=  \E\left\langle y_{k+1}-y^*(\theta_k), y^*(\theta_k)-y^*(\theta_{k+1})- (\nabla y^*(\theta_k))^{\top}(\theta_{k+1}-\theta_k)\right\rangle\\
    &\qquad +  \E\left\langle y_{k+1}-y^*(\theta_k), (\nabla y^*(\theta_k))^{\top}(\theta_{k+1}-\theta_k)\right\rangle  \\ 
   &\stackrel{(a)}{\leq}    \frac{L_{y,2}^2}{2} \E  \|y_{k+1}-y^*(\theta_k)\|_2  \|\theta_{k+1}-\theta_k\|_2^2 + \E\left[\left\langle y_{k+1}-y^*(\theta_k), \E[(\nabla y^*(\theta_k))^{\top}(\theta_{k+1}-\theta_k)\mid y_{k+1}]\right\rangle\right]\\ 
    &\stackrel{(b)}{\leq}   \frac{L_{y,2}^2}{2} \E  \|y_{k+1}-y^*(\theta_k)\|_2  \|\theta_{k+1}-\theta_k\|_2^2 +  \alpha_kL_{y} \E\left\| y_{k+1}-y^*(\theta_k)\right\| \left\|\bar{h}_f(\theta_k, y_{k+1})\right\| \\
    &\stackrel{(c)}{\leq}  \frac{L_{y,2}^2}{4} \E  \|y_{k+1}-y^*(\theta_k)\|_2^2  \|\theta_{k+1}-\theta_k\|_2^2+\frac{L_{y,2}^2}{4}    \E\|\theta_{k+1}-\theta_k\|_2^2 +  \alpha_k L_{y} \E\left\| y_{k+1}-y^*(\theta_k)\right\| \left\|\bar{h}_f(\theta_k, y_{k+1})\right\| \\
    &\stackrel{(d)}{\leq}   \frac{\alpha_k^2C_f^2L_{y,2}^2}{2} \E  \|y_{k+1}-y^*(\theta_k)\|_2  +\frac{L_{y,2}^2}{4}    \E\|\theta_{k+1}-\theta_k\|_2^2 +\alpha_k L_{y,2}^2 \E  \|y_{k+1}-y^*(\theta_k)\|_2^2  +\frac{\alpha_k}{4}\E\left\|\bar{h}_f(\theta_k, y_{k+1})\right\|^2 \\ 
    &\leq  \left(\alpha_k L_{y,2}^2  +\frac{\alpha_k^2C_f^2L_{y,2}^2}{4}\right) \E  \|y_{k+1}-y^*(\theta_k)\|_2^2 +\frac{\alpha_k }{4}\E\left\|\bar{h}_f(\theta_k, y_{k+1})\right\|^2+\frac{\alpha_k^2C_f^2L_{y,2}^2}{4}    \numberthis
\end{align*}}
where (a) follows from the $L_{y,2}$-smoothness of $y^*$ with respect to $\theta$; (b) follows from $L_{y}$ is the Lipschitz constant of $y^*$ in Proposition \ref{proposition:omega-lipschitz} and 
$$\E[(\nabla y^*(\theta_k))^{\top}(\theta_{k+1}-\theta_k)\mid y_{k+1}]=\nabla y^*(\theta_k))^{\top}\bar{h}_f(\theta_k, y_{k+1});$$ 
(c) uses the Young's inequality; (d) uses  the Young's inequality and the fact that $\|\theta_{k+1}-\theta_k\|_2=\alpha_k\|h_f(\xi_k',\theta_k, y_{k+1})\|\leq C_g C_\psi = C_f$ and $\|\bar{h}_f(\theta_k, y_{k+1})\|\leq C_f$.

We bound 
\begin{align*}\label{eq.lipw}
	\E\left\| y^*(\theta_k)-y^*(\theta_{k+1})\right\|_2^2& \leq L_{y}^2\E\left\| \theta_k-\theta_{k+1}\right\|_2^2\\
	&\leq L_{y}^2 \alpha_k^2 \E\left\| \hat{\delta}(\xi_k,y_{k})\psi_{\theta_{k}}(s_k,a_k)\right\|_2^2 \leq L_{y}^2 C_f^2 \alpha_k^2 \numberthis
\end{align*}
where the inequality is due to the $L_y$-Lipschitz of $y^*(\theta)$ shown in Proposition \ref{proposition:omega-lipschitz}, and the last inequality follows the fact that
\begin{align*}\label{eq:C_p}
    \| \hat{\delta}(\xi_k, y_{k})\psi_{\theta_{k}}(s_k,a_k) \|_2
\leq C_g C_\psi = C_f. \numberthis
\end{align*}

Substituting \eqref{eq:tmp8}-\eqref{eq.lipw} into (\ref{eq:tmp1}) yields
\begin{align*} \label{eq:async-tts-tmp0}
   \E\|y_{k+1}-y^*(\theta_{k+1})\|_2^2
    &\leq  \left(1+\alpha_k L_{y,2}^2  +\frac{\alpha_k^2C_f^2L_{y,2}^2}{4}\right)\E\left[\left\|y_{k+1}-y^*(\theta_k)\right\|_2^2 \right] \\
    &~~~~ +\frac{\alpha_k}{4}\E\left\|\bar{h}_f(\theta_k, y_{k+1})\right\|^2+\frac{\alpha_k^2C_f^2L_{y,2}^2}{4}+L_{y}^2 C_f^2 \alpha_k^2.\numberthis
\end{align*}
 \hfill\myQED

\subsection{Proof of Theorem \ref{theorem2}}\label{section:async-tts-actor-proof}
Recall the notations:
\begin{align*}
    \hat{\delta}(\xi,y) &\coloneqq r(s,a,s') + \gamma\phi(s')^\top y - \phi(s)^\top y, \\
     \bar{\delta}(\xi,y) &\coloneqq \E_{\substack{s \sim d_{\theta}, a \sim \pi_\theta, s' \sim \mathcal{P}}} \left[r(s,a,s') + \gamma\phi(s')^\top y - \phi(s)^\top y \mid y\right] \\ 
    \delta(\xi,\theta) &\coloneqq r(s,a,s') + \gamma V_{\pi_\theta}(s') - V_{\pi_\theta}(s).
\end{align*}
The actor update can be written compactly as:
\begin{align}\label{update:async-tts-actor-update}
\theta_{k+1} = \theta_k + \alpha_k h_f(\xi_k',\theta_k, y_{k+1})
\end{align}
where $h_f(\xi_k',\theta_k, y_{k+1}):=\hat{\delta}(\xi_k',y_{k+1})\psi_{\theta_k}(s_k,a_k)$. Define $\bar{h}_f(\theta_k, y_{k+1}):=\E[\hat{\delta}(\xi_k',y_{k+1})\psi_{\theta_k}(s_k,a_k)|y_{k+1}]$.
Then we are ready to give the convergence proof.

\textit{Proof.}
From $L_F$-Lipschitz of policy gradient in Proposition \ref{prop:Lj-lip}, taking expectation conditioned on $\theta_k, y_{k+1}$, we have:
{\small\begin{align*}\label{pf.ac_smth}
   &\E [F(\theta_{k+1})]-F(\theta_k) \numberthis\\
    &\geq  \E\left\langle \nabla F(\theta_k), \theta_{k+1}-\theta_k \right\rangle - \frac{L_F}{2}\E\|\theta_{k+1}-\theta_k\|_2^2 \\
       &\geq  \alpha_k\E\left\langle \nabla F(\theta_k), \bar{h}_f(\theta_k, y_{k+1})\right\rangle - \frac{L_F}{2}\E\|\theta_{k+1}-\theta_k\|_2^2 \\ 
    &=   \frac{\alpha_k}{2}\E\left\|\nabla F(\theta_k)\right\|^2+ \frac{\alpha_k}{2}\E\left\|\bar{h}_f(\theta_k, y_{k+1})\right\|^2- \frac{\alpha_k}{2}\E\left\| \nabla F(\theta_k)- \bar{h}_f(\theta_k, y_{k+1})\right\|^2 - \frac{L_F}{2}\E\|\theta_{k+1}-\theta_k\|_2^2 \\
     &\geq   \frac{\alpha_k}{2}\E\left\|\nabla F(\theta_k)\right\|^2+ \frac{\alpha_k}{2}\E\left\|\bar{h}_f(\theta_k, y_{k+1})\right\|^2- \frac{\alpha_k}{2}\E\left\| \nabla F(\theta_k)- \bar{h}_f(\theta_k, y_{k+1})\right\|^2\\
     &\quad - \frac{L_F\alpha_k^2}{2}\E\|\bar{h}_f(\theta_k, y_{k+1})\|_2^2- \frac{L_F\alpha_k^2}{2}\E\|\bar{h}_f(\theta_k, y_{k+1})-h_f(\xi_k',\theta_k, y_{k+1})\|_2^2 \\ 
      &\geq  \frac{\alpha_k}{2}\E\left\|\nabla F(\theta_k)\right\|^2+ \left(\frac{\alpha_k}{2}- \frac{L_F\alpha_k^2}{2}\right) \E\left\|\bar{h}_f(\theta_k, y_{k+1})\right\|^2- \frac{\alpha_k}{2}\E\left\| \nabla F(\theta_k)- \bar{h}_f(\theta_k, y_{k+1})\right\|^2 - \frac{L_FC_f^2\alpha_k^2}{2} 
\end{align*}}
where the last inequality follows the definition of $C_f$ in (\ref{eq:C_p}).

We next bound the gradient bias as
{\small\begin{align*}\label{pf.grad_bias}
    \left\| \nabla F(\theta_k)- \bar{h}_f(\theta_k, y_{k+1})\right\|^2 &= \left\| \nabla F(\theta_k)- \E[\hat{\delta}(\xi_k',y_{k+1})\psi_{\theta_k}(s_k,a_k)|y_{k+1}] \right\|^2\\
    &\leq 2  \left\| \nabla F(\theta_k)- \E[\hat{\delta}(\xi_k',y^*(\theta_k))\psi_{\theta_k}(s_k,a_k)|y_{k+1}]\right\|^2\\
    &\quad + 2 \left\|\E[(\hat{\delta}(\xi_k',y^*(\theta_k))-\hat{\delta}(\xi_k',y_{k+1}))\psi_{\theta_k}(s_k,a_k)| y_{k+1}]\right\|^2 \\
    &\leq 4  \underbracket{\left\| \nabla F(\theta_k)- \E[\delta(\xi_k', \theta_k) \psi_{\theta_k}(s_k,a_k)|y_{k+1}]\right\|^2}_{I_1}\\
   &\quad +4  \underbracket{\left\|\E[\delta(\xi_k', \theta_k) \psi_{\theta_k}(s_k,a_k)|y_{k+1}]- \E[\hat{\delta}(\xi_k',y^*(\theta_k))\psi_{\theta_k}(s_k,a_k)|y_{k+1}]\right\|^2}_{I_2}\\
    &\quad + 2 \underbracket{\left\|\E[(\hat{\delta}(\xi_k',y^*(\theta_k))-\hat{\delta}(\xi_k',y_{k+1}))\psi_{\theta_k}(s_k,a_k)|y_{k+1}]\right\|^2}_{I_3}. \numberthis
\end{align*}}

Then we bound $I_1$ as
\begin{align*}
   I_1&=\left\| \nabla F(\theta_k)- \E[\delta(\xi_k', \theta_k) \psi_{\theta_k}(s_k,a_k)|\theta_k, y_{k+1}]\right\|^2\\
    &= \left\| \nabla F(\theta_k)- \E_{\substack{s_k \sim d_{\theta_k}\\ a_k \sim \pi_{\theta_k}, s'_k \sim \mathcal{P}}}\left[ \left(r(s_k,a_k,s'_k) + \gamma V_{\pi_{\theta_k}}(s'_k)-  V_{\pi_{\theta_k}}(s_k) \right)\psi_{\theta_k}(s_k,a_k) \bigg| \theta_k, y_{k+1} \right]\right\|^2 \\
      &= \left\| \nabla F(\theta_k)- \E_{\substack{s_k \sim d_{\theta_k}\\ a_k \sim \pi_{\theta_k}}}\left[ A_{\pi_{\theta_k}}(s_k,a_k)\psi_{\theta_k}(s_k,a_k) \bigg| \theta_k, y_{k+1} \right]\right\|^2 =0
\end{align*}
where the last equality follows from the policy gradient theorem. 

Then we bound $I_2$ as
\begin{align*}
   I_2&=\left\|\E[\delta(\xi_k', \theta_k) \psi_{\theta_k}(s_k,a_k)|\theta_k, y_{k+1}]- \E[\hat{\delta}(\xi_k',y^*(\theta_k))\psi_{\theta_k}(s_k,a_k)|\theta_k, y_{k+1}]\right\|^2\\
    &=\left\|\E\left[\left(\delta(\xi_k', \theta_k)  - \hat{\delta}(\xi_k',y^*(\theta_k))\right)\psi_{\theta_k}(s_k,a_k)|\theta_k, y_{k+1}\right]\right\|^2\\\
       &=\left\|\E\left[\left|\delta(\xi_k', \theta_k)  - \hat{\delta}(\xi_k',y^*(\theta_k))\right|\|\psi_{\theta_k}(s_k,a_k)\||\theta_k, y_{k+1}\right]\right\|^2\\\ 
            &=  C_\psi^2 \left\|\E\left[\left|\delta(\xi_k', \theta_k)  - \hat{\delta}(\xi_k',y^*(\theta_k))\right| |\theta_k, y_{k+1}\right]\right\|^2\\\   
           &\leq C_\psi^2 \left(\gamma \E\left|\phi(s'_k)^\top y^*(\theta_k)-V_{\pi_{\theta_k}}(s'_k)\right| + \E\left|V_{\pi_{\theta_k}}(s_k) - \phi(s_k)^\top y^*(\theta_k) \right| \right) \\
    &\leq   C_\psi^2 \bigg(\gamma \sqrt{\E\left|\phi(s'_k)^\top y^*(\theta_k)-V_{\pi_{\theta_k}}(s'_k)\right|^2} + \sqrt{\E\left|V_{\pi_{\theta_k}}(s_k) - \phi(s_k)^\top y^*(\theta_k) \right|^2} \bigg) \\     
      &\leq C_\psi^2(1+\gamma)\epsilon_{app}.
\end{align*}

 Then we bound $I_3$ as
\begin{align*}
   I_3&=\left\|\E[(\hat{\delta}(\xi_k',y^*(\theta_k))-\hat{\delta}(\xi_k',y_{k+1}))\psi_{\theta_k}(s_k,a_k)|\theta_k, y_{k+1}]\right\|^2\\
    &\leq C_\psi^2\E\left[\| \hat{\delta}(\xi_k',y^*(\theta_k))-\hat{\delta}(\xi_k',y_{k+1})\|^2|\theta_k, y_{k+1}\right]\\
      &=   C_\psi^2\E\left[\|\gamma\phi(s'_k)^\top y^*(\theta_k) - \phi(s_k)^\top y^*(\theta_k)-\gamma\phi(s'_k)^\top y_{k+1}+\phi(s_k)^\top y_{k+1}\|^2|\theta_k, y_{k+1}\right]\\
      &\leq C_\psi^2(1+\gamma)\|y^*(\theta_k)-y_{k+1}\|^2.
\end{align*}

Then \eqref{pf.grad_bias} can be rewritten as
\begin{align*}
  \left\| \nabla F(\theta_k)- \bar{h}_f(\theta_k, y_{k+1})\right\|^2 \leq 4C_\psi^2(1+\gamma)\epsilon_{app}+2C_\psi^2(1+\gamma)\|y^*(\theta_k)-y_{k+1}\|^2
\end{align*}
plugging which into \eqref{pf.ac_smth} leads to 
\begin{align*}\label{pf.ac_smth2}
   \E [F(\theta_{k+1})]
      &\geq F(\theta_k) + \frac{\alpha_k}{2}\E\left\|\nabla F(\theta_k)\right\|^2+ \left(\frac{\alpha_k}{2}- \frac{L_F\alpha_k^2}{2}\right) \E\left\|\bar{h}_f(\theta_k, y_{k+1})\right\|^2\\
      &\quad - \frac{\alpha_k}{2}\E\left\| \nabla F(\theta_k)- \bar{h}_f(\theta_k, y_{k+1})\right\|^2 - \frac{L_FC_f^2\alpha_k^2}{2}\\
      &\geq F(\theta_k) + \frac{\alpha_k}{2}\E\left\|\nabla F(\theta_k)\right\|^2+ \left(\frac{\alpha_k}{2}- \frac{L_F\alpha_k^2}{2}\right) \E\left\|\bar{h}_f(\theta_k, y_{k+1})\right\|^2  - \frac{L_FC_f^2\alpha_k^2}{2}\\
      &\quad       -  2\alpha_k C_\psi^2(1+\gamma)\epsilon_{app} - \alpha_k C_\psi^2(1+\gamma)\|y^*(\theta_k)-y_{k+1}\|^2.
       \numberthis
\end{align*}

Consider the difference of the Lyapunov function $\mathbb{V}^k := -F(\theta_k) + \|y_k-y^*(\theta_k)\|_2^2$, given by
{\small\begin{align*}
\E [\mathbb{V}^{k+1}]-\E [\mathbb{V}^k]=&
	- \E [F(\theta_{k+1})]+\E\|y_{k+1}-y^*(\theta_{k+1})\|_2^2+	 \E [F(\theta_k)]-\E\|y_k-y^*(\theta_k)\|_2^2\\
	 \leq &- \frac{\alpha_k}{2}\E\left\|\nabla F(\theta_k)\right\|^2- \left(\frac{\alpha_k}{2}- \frac{L_F\alpha_k^2}{2}\right) \E\left\|\bar{h}_f(\theta_k, y_{k+1})\right\|^2 + \frac{L_FC_f^2\alpha_k^2}{2} + 2\alpha_k C_\psi^2(1+\gamma)\epsilon_{app}\\
      &       + \alpha_k C_\psi^2(1+\gamma)\|y^*(\theta_k)-y_{k+1}\|^2+\E\|y_{k+1}-y^*(\theta_{k+1})\|_2^2 -\E\|y_k-y^*(\theta_k)\|_2^2\\
    	 \leq &- \frac{\alpha_k}{2}\E\left\|\nabla F(\theta_k)\right\|^2- \left(\frac{\alpha_k}{2}-\frac{\alpha_k  }{4}- \frac{L_F\alpha_k^2}{2}\right) \E\left\|\bar{h}_f(\theta_k, y_{k+1})\right\|^2  +\frac{L_FC_f^2\alpha_k^2}{2} \\
      & +\left(1+\alpha_kL_{y,2}^2  +\frac{\alpha_k^2C_f^2L_{y,2}^2}{4}+  \alpha_k C_\psi^2(1+\gamma)\right)\E\left[\left\|y_{k+1}-y^*(\theta_k)\right\|_2^2 \right] \\
    &-  \E\left\|y_k - y^*(\theta_k) \right\|_2^2  +\frac{\alpha_k^2C_f^2L_{y,2}^2}{4}+L_{y}^2 C_f^2 \alpha_k^2+ 2\alpha_k C_\psi^2(1+\gamma)\epsilon_{app}.\numberthis
\end{align*}}

Applying \eqref{eq:tmp2-2} to bound $\E\left[\left\|y_{k+1}-y^*(\theta_k)\right\|_2^2 \right]$, we have
\begin{align*}\label{eq.106}
&\E [\mathbb{V}^{k+1}]-\E [\mathbb{V}^k]
     \\
    	 \leq	 &- \frac{\alpha_k}{2}\E\left\|\nabla F(\theta_k)\right\|^2- \left(\frac{\alpha_k}{2}-\frac{\alpha_k  }{4}- \frac{L_F\alpha_k^2}{2}\right) \E\left\|\bar{h}_f(\theta_k, y_{k+1})\right\|^2  +\frac{L_FC_f^2\alpha_k^2}{2} + 2\alpha_k C_\psi^2(1+\gamma)\epsilon_{app}\\
      &  +\left[\Big(1+\alpha_kL_{y,2}^2  +\frac{\alpha_k^2C_f^2L_{y,2}^2}{4}+  \alpha_k C_\psi^2(1+\gamma)\Big)(1-2\lambda \beta_k)-1\right] \E\left\|y_k - y^*(\theta_k) \right\|_2^2 \\
    &  +\left(1+\alpha_kL_{y,2}^2  +\frac{\alpha_k^2C_f^2L_{y,2}^2}{4}+  \alpha_k C_\psi^2(1+\gamma)\right) C_g^2\beta_k^2+\frac{\alpha_k^2C_f^2L_{y,2}^2}{4}+L_{y}^2 C_f^2 \alpha_k^2.\numberthis
\end{align*}

Similar to the steps   \eqref{eq.step-cond}-\eqref{eq.step-cond-3}, if we select
\begin{equation}\label{eq.step-cond-4}
    \alpha_k=\min\left\{\frac{1}{2L_F}, \frac{\alpha}{\sqrt{K}}\right\},~~~~~~~\beta_k=\frac{4L_{y,2}^2  +  8 C_\psi^2+  {C_f^2L_{y,2}^2}/{2L_F}}{8\lambda}\alpha_k.
\end{equation}
which ensures that 
\begin{subequations}
\begin{align}
	\frac{\alpha_k  }{4}- \frac{L_F\alpha_k^2}{2}\geq 0\\
	\Big(1+\alpha_kL_{y,2}^2  +  \alpha_k C_\psi^2(1+\gamma)+\frac{\alpha_k^2C_f^2L_{y,2}^2}{4}\Big)(1-2\lambda \beta_k)\leq 1
\end{align}	
\end{subequations}
we can simplify \eqref{eq.106} as 
\begin{align*}
\E [\mathbb{V}^{k+1}]-\E [\mathbb{V}^k]
    	 \leq &- \frac{\alpha_k}{2}\E\left\|\nabla F(\theta_k)\right\|^2   +\frac{L_FC_f^2\alpha_k^2}{2} + 2\alpha_k C_\psi^2(1+\gamma)\epsilon_{app}+\frac{\alpha_k^2C_f^2L_{y,2}^2}{4}\\
    & +\left(1+\alpha_kL_{y,2}^2  +\frac{\alpha_k^2C_f^2L_{y,2}^2}{4}+  \alpha_k C_\psi^2(1+\gamma)\right) C_g^2\beta_k^2+L_{y}^2 C_f^2 \alpha_k^2.\numberthis
\end{align*}

After telescoping, we have
\begin{align*}\label{eq:asymp-equiv-actor-iid}
 \frac{1}{K}\sum_{k=1}^K\E\left\|\nabla F(\theta_k)\right\|^2 
    	 \leq &\frac{2 \mathbb{V}^1 }{\alpha_k K}  + L_FC_f^2\alpha_k + 4 C_\psi^2(1+\gamma)\epsilon_{app}+\frac{\alpha_k C_f^2L_{y,2}^2}{2}+2L_{y}^2 C_f^2 \alpha_k \\
    &  +2\left(1+\alpha_k L_{y,2}^2  +\frac{\alpha_k^2C_f^2L_{y,2}^2}{4}+  \alpha_k C_\psi^2(1+\gamma)\right)\frac{C_g^2\beta_k^2}{\alpha_k} \numberthis
\end{align*}
which, together with $\alpha_k={\cal O}(1/\sqrt{K}), \beta_k={\cal O}(1/\sqrt{K})$, completes the proof. \hfill \myQED

\end{document}